\documentclass[11pt]{article}
\pdfoutput=1
\usepackage{pdfsync}
\usepackage[OT1]{fontenc}
\usepackage[dvipsnames]{xcolor}
\usepackage{smile}
\usepackage[colorlinks,
            linkcolor=red,
            anchorcolor=blue,
            citecolor=blue
            ]{hyperref}
\usepackage{mathrsfs}
\usepackage{fullpage}
\usepackage[protrusion=true,expansion=true]{microtype}
\usepackage{pbox}
\usepackage{setspace}
\usepackage{tabularx}
\usepackage{float}
\usepackage{wrapfig,lipsum}
\usepackage{microtype}
\usepackage{graphicx}

\usepackage{enumerate}
\usepackage{enumitem}
\usepackage{pgfplots}
\usetikzlibrary{arrows,shapes,snakes,automata,backgrounds,petri}
\usepackage{booktabs} 

\usepackage{mathtools}
\allowdisplaybreaks

\makeatletter
\newtheorem*{rep@theorem}{\rep@title}
\newcommand{\newreptheorem}[2]{%
\newenvironment{rep#1}[1]{%
 \def\rep@title{#2 \ref{##1}}%
 \begin{rep@theorem}}%
 {\end{rep@theorem}}}
\makeatother

\theoremstyle{plain}
\newtheorem{theorem}{Theorem}[section]
\newtheorem{proposition}[theorem]{Proposition}
\newtheorem{lemma}[theorem]{Lemma}
\newreptheorem{theorem}{Theorem}
\newreptheorem{lemma}{Lemma}

\theoremstyle{definition}
\newtheorem{definition}[theorem]{Definition}
\newtheorem{assumption}[theorem]{Assumption}

\newreptheorem{definition}{Definition}
\theoremstyle{remark}
\newtheorem{remark}[theorem]{Remark}

\title{\huge Learn to Match with No Regret: 
Reinforcement Learning in Markov Matching Markets}

\author{Yifei Min\thanks{Yale University. E-mail: {\tt yifei.min@yale.edu}} \qquad\qquad ~
    Tianhao Wang\thanks{Yale University. E-mail: {\tt tianhao.wang@yale.edu}} \qquad\qquad ~
	Ruitu Xu\thanks{Yale University. E-mail: {\tt ruitu.xu@yale.edu}}\\
	Zhaoran Wang\thanks{Northwestern University. E-mail: {\tt zhaoranwang@gmail.com}}\qquad\qquad
	Michael I. Jordan\thanks{University of California, Berkeley. E-mail: {\tt jordan@cs.berkeley.edu }}\qquad\qquad
	Zhuoran Yang\thanks{Yale University. E-mail: {\tt zhuoran.yang@yale.edu}}
}
\date{}

\begin{document}

\maketitle

\begin{abstract}
We study a Markov matching market involving a planner and a set of strategic agents on the two sides of the market.
At each step, the agents are presented with a dynamical context, where the contexts determine the utilities. 
The planner controls the transition of the contexts to maximize the cumulative social welfare, while the agents aim to find a myopic stable matching at each step. Such a setting captures a range of applications including ridesharing platforms. We formalize the problem by proposing a reinforcement learning framework that integrates optimistic value iteration with maximum weight matching. 
The proposed algorithm addresses the coupled challenges of sequential exploration, matching stability, and function approximation. We prove that the algorithm achieves sublinear regret. 
\end{abstract}

\section{Introduction}\label{sec: intro}

Large-scale digital markets play a crucial role in modern economies. Increasingly, the understanding and design of such markets require tools from both economics and machine learning.  Indeed, dynamically changing market environments require mechanisms to be adaptive, scalable, and incentive compatible in the face of significant nonstationarity.  The data streams that arise from digital markets provide opportunities to cope with such challenges, via learning-based mechanism design.
Recent work \citep{jagadeesan2021learning,liu2021bandit,sankararaman2021dominate,basu2021beyond} has begun to apply modern machine learning tools to problems in adaptive mechanism design. One particular area of focus in learning-aware market design has been matching, a class of problems central to microeconomics \citep{mas1995microeconomic}. Existing work has focused on static matching markets, however, and the more challenging yet critically important setting of dynamic matching markets has been neglected.  Such markets are our focus in the current paper, where we provide reinforcement learning (RL) techniques to address dynamic matching market problems.

Specifically, we propose a \emph{Markov matching market} model involving a planner and a set of two-sided agents.
Given the time horizon $H\in\ZZ_+$, a set of agents $I_h\cup J_h$ enter the market in context $C_h$ at step $h\in[H]$, with unknown utility functions $u_h$ and $v_h$ for agents in $I_h$ and $J_h$ respectively. The agents' utilities $u_h$ and $v_h$ depend on the context $C_h$.
Given the utility functions, the agents seek to achieve a myopic stable matching with transferable utilities \citep{shapley1971assignment}.
In particular, the contexts $C_h$ are subject to a Markov transition kernel which is controlled by the planner's policy. The goal of the planner is to select an optimal policy, which together with the stable matching in each step maximizes the expected accumulated social welfare.

As an illustration, consider a simplified abstraction of ride-hailing platforms  \citep{qin2020ride, ozkan2020dynamic,hu2021dynamic}, where the horizon is set as the time span of a day. 
In this case, the platform is the planner, and the two-sided agents $I_h$ and $J_h$ are the drivers and the riders.
The contexts may include information such as  GPS location, car types, and pricing.
The platform is required to ensure the stability of the matching at each step so that the agents do not prefer alternative outcomes, otherwise they may leave the platform.
In addition, the platform needs to adjust the policy to maximize the accumulated social welfare as a measure of the satisfaction of both drivers and riders.

As illustrated in this example, there are several key challenges in Markov matching markets.
First, the agents' utilities and the transition of contexts are unknown, so we need to perform efficient exploration to collect the required information.
Second, we need to ensure the stability of the matching in each step, for which a consideration of only the difference in total utilities is insufficient, as discussed in \citet{jagadeesan2021learning}.
To this end, we adopt a notion of \emph{Subset Instability} (SI) from \citet{jagadeesan2021learning} as a metric that quantifies the distance between a proposed matching and the optimal stable matching. This metric has the flavor of Shapley value, comparing the discrepancy between the total utilities over all subsets of the participating agents, while accounting for transfers.  For efficient estimation, we need to take into account the function class of the utilities as well as the Markov transition kernel, thus demanding a systematic usage of function approximation.

To tackle these challenges, we develop a novel algorithm called \emph{Sequential Optimistic Matching} (SOM), which features a combination of optimistic value iteration and max-weight matching.
Note that the planner's problem becomes a standard Markov decision process (MDP) if we regard the value of the max-weight matching over the true utilities as the reward.
Inspired by this observation, on the agents side, our algorithm applies the optimism principle to construct UCB estimates of the utilities.
Based on these UCB estimates, the algorithm computes the corresponding max-weight matching, and the values of the resulting matching serve as the surrogate rewards.
Although we do not have access to unbiased estimate of the rewards, the key property here is that the surrogate rewards upper bound the true social welfare, thus justifying optimistic planning by the planner.
Interestingly, our framework can readily incorporate any online RL algorithm based on optimism.

On the theoretical side, we show that the suboptimality of the accumulated social welfare for our algorithm consists of two parts: (1) the planner's regret in terms of the suboptimality of the policy, and (2) the agents' regret in terms of SI.
\citet{jagadeesan2021learning} proved that SI can be bounded by the sum of the optimistic bonuses, and we further show that the planner's regret can also be bounded by the bonus sum.
In this way, we reconcile the seemingly independent learning goals of the planner and the agents, and thereby provide a unified approach to controlling the suboptimality of the total social welfare.
In particular, based on the above decomposition, we further show that in the case of linear function approximation, our algorithm enjoys a sublinear regret independent of the size of the context space.

Compared with existing methods on online RL such as \texttt{LSVI-UCB} \citep{jin2020provably}, our framework incorporates a matching problem in each time step. Compared with the setup of matching bandits~\citep{jagadeesan2021learning}, our model provides an extension to a dynamic setup with transitions between contexts. More importantly, our work lies beyond a straightforward combination of online RL with matching bandits due to the unique technical challenges that we will discuss in Section~\ref{sec: proof sketch}.  

The main contributions are summarized as follows:
\begin{itemize}
    \item We propose a novel Markov matching market model that captures a range of instances of centralized matching problems.
    
    \item We develop a novel algorithm that combines optimistic value iteration with max-weight matching, such that any online RL algorithm based on optimism can be readily incorporated into the framework.
    
    \item We provide a general analysis framework and show that our proposed algorithms achieve sublinear regret under proper structural assumptions.
\end{itemize}

\section{Related Work}

There is an emerging line of research on learning stable matchings with bandit feedback \citep{das2005two,liu2020competing,liu2021bandit,sankararaman2021dominate,cen2021regret,basu2021beyond} using the mature tools from the bandit literature.
Most of them focus on matchings with non-transferable utilities \citep{gale1962college}, which fails to capture real-world markets with monetary transfers between agents, e.g., payments from passengers to drivers on ride-hailing platforms.
The study of learning for matchings with transferable utilities is comparably limited, and our work extends \citet{jagadeesan2021learning} to dynamic scenarios in this regime. 
Broadly speaking, our work is also related to the research on learning economic models via RL. 
In particular, \citet{kandasamy2020mechanism,rasouli2021data} studied VCG mechanisms, and \citet{guo2021online} study exchange economies. 
Although similar in spirit, these models differ from our Markov matching markets in their mathematical structure, and thus have different solution concepts and planning methods.

Our model of Markov matching markets is related to the topic of dynamic matching in the economics literature \citep{taylor1995long, satterthwaite2007dynamic, niederle2009decentralized,unver2010dynamic,anderson2014dynamic, lauermann2014stable,leshno2019dynamic, akbarpour2020thickness, baccara2020optimal, loertscher2018optimal, doval2019efficiency}.
Instead of studying the learning problem in matching markets, the goal therein is mainly focused on the problem of optimal mechanism design \citep{akbarpour2014dynamic} with known utilities or explicit modelling of agents' arrivals via queuing models \citep{zenios1999modeling, gurvich2015dynamic}.
In this work we focus on the notion of (static) stability of myopic matchings \citep{shapley1971assignment}.
There is also a line of literature on the notion of dynamic stability for matching markets \citep{damiano2005stability, doval2014theory, kadam2018multiperiod, doval2019dynamically,kotowski2019perfectly, kurino2020credibility,liu2020stability},
and it is an interesting open problem to study learning for dynamically stable matchings.

Our methodology builds upon recent progress in online RL, where the ``optimism in face of uncertainty'' principle has engendered efficient algorithms that are either model-based \citep{jaksch2010near, osband2016generalization, azar2017minimax, dann2017unifying} or model-free \citep{strehl2006pac,jin2018q,fei2020risk,fei2020exponential}, and can be combined with function approximation techniques \citep{yang2019sample, jin2020provably, zanette2020learning, ayoub2020model, wang2020reinforcement,fei2021risk,yang2020function,  zhou2021nearly, min2021learning, min2021variance, du2021bilinear, jin2021bellman}.
We note that these approaches can be incorporated into our framework with proper modifications on structural assumptions and correspondingly the algorithm.
We do not pursue these extensions here, however, as our regret analysis is already sufficiently challenging given the lack of an unbiased estimate of the reward, and the additional constraints imposed by the requirements of matching stability.

\section{Preliminaries}

\subsection{Markov Matching Markets}\label{sec:pre_MMM}
We first review basic concepts for matching with transferable utilities \citep{shapley1971assignment}.
Consider a two-sided matching market where $\cI$ denotes the set of all side-1 agents (e.g., buyers) and $\cJ$ denotes the set of all side-2 agents (e.g., sellers). 
Given any set of participating agents $I\times J \in 2^{\cI}\times 2^{\cJ}$, a matching $X$ is a set of pairs $(i,j)$ indicating $i\in I$ is matched with $j\in J$, and each agent can be matched at most once. 
For any pair of agents $(i,j)\in I\times J$, we denote by $u(i,j)$ the utility of agent $i$ and $v(i,j)$ the utility of agent $j$ when they are matched. 
In addition to the matching $X$, we also allow transfers between agents, summarized by the transfer function $\tau: I \cup J \to \RR$. 
For each agent $i \in I \cup J$, $\tau(i)$ is the transfer that it receives. 
We assume that the transfers are within agents, which implies that $\sum_{i \in I\cup J} \tau(i) = 0$.

The overall market outcome is denoted by a tuple $(X,\tau)$, where $X$ represents the matching and $\tau$ represents transfers. 
For any $(i,j)\in X$, the total utilities are $u(i,j)+\tau(i)$ and $v(i,j) + \tau(j)$ for $i$ and $j$ respectively.
Moreover, if no agents prefer any alternate outcome, then we say $(X,\tau)$ is stable (see Definition \ref{def: stable matching} for details).
The stable matching can be found by solving the corresponding max-weight matching, as will be explained in Section \ref{sec:reward_estimation}. 

Based on these (classical) definitions, we formulate the notion of a \emph{Markov matching market} involving a planner and a set of two-sided agents.
Throughout the paper, we focus on matchings with transferable utilities between two-sided agents, and we may omit such descriptions for convenience.

\begin{definition}[Markov matching markets]\label{def:MMM}
    A Markov matching market is denoted by a tuple $M = (\cC, \Upsilon, \{I_h\}_{h=1}^H, \{J_h\}_{h=1}^H, \{\PP_h\}_{h=1}^H, \{u_h\}_{h=1}^H, \{v_h\}_{h=1}^H)$.
    Here $\cC$ is the set of contexts and $\Upsilon$ denotes the set of planner's actions.
    At each step $h\in[H]$, $I_h \cup J_h \subset \cI\times\cJ$ is the set of participating agents, and $u_h: \cC \times \Upsilon \times \cI \times \cJ \to \RR$ and $v_h:\cC \times \Upsilon \times \cI \times \cJ \to \RR$ are utility functions for two sides of agents respectively.
    For each $h\in[H]$, $\PP_h(C'\mid C, e)$ is the transition probability for context $C$ to transit to $C'$ given action $e$. 
\end{definition}

In such Markov matching markets, the learning goal is two-fold: (1) to learn the stable matching in each step and (2) to maximize the accumulated social welfare. 
We now translate the problem into the language of RL.

\subsection{A Reinforcement Learning Approach to Markov Matching Markets}
Given Definition \ref{def:MMM}, we consider an episodic setting with $K$ episodes where each episode consists of $H$ steps of sequential matchings. 
Each episode proceeds in the following way: at each step $h\in[H]$ under context $C_h$, a set of agents $I_h\cup J_h$ enter the market.
The planner takes action $e_h$, implement the matching $(X_h,\tau_h)$, and observes the noisy feedback of utilities $u_h(C_h,e_h,i,j)$ and $v_h(C_h,e_h,i,j)$ for all $(i,j)\in X_h$.
Then the context transitions according to $\PP_h(\cdot\mid C_h,e_h)$, and the market proceeds to the next step.

Note that here the implemented matching in each step is myopic and we seek for stability for each of these matchings.
The maximization of the accumulated social welfare is achieved through the planner's action $\{e_h\}_{h=1}^H$ that controls the transitions of contexts, which together with planner's actions determine the optimal values of the matchings.
To apply RL to maximize the accumulated social welfare, let us specify the ingredients of the corresponding RL problem.

\paragraph{States and Actions.} 
The state space is $\cS = \cC\times 2^\cI\times 2^\cJ$.
The action space is $\cA = \Upsilon\times\cX\times\cT$ where $\cX$ denotes the set of all matchings and $\cT$ is the set of all possible transfers among the agents.
For each step $h\in[H]$, the state $s_h = (C_h, I_h, J_h)$ contains the context and participating agents, and the action $a_h = (e_h, X_h, \tau_h)$ contains the planner's action $e_h$ and the matching outcome $(X_h,\tau_h)$ for the agents.

\paragraph{Rewards.}
At each step $h\in[H]$, given the state $s_h = (C_h, I_h, J_h)$ and the action $a_h=(e_h,X_h,\tau_h)$, the immediate reward is the social welfare (i.e. sum of utilities): 
\begin{align}\label{eq: immediate reward}
    r_h(s_h, a_h) \coloneqq \sum_{(i,j)\in X_h} [u_h(C_h,e_h,i,j) + v_h(C_h,e_h,i,j)].
\end{align} 
Note that the transfer $\tau_h \in \cT$ does not appear in the reward since the total transfer sums to zero. 

\paragraph{Transition of States.}
The state consists of the context in $\cC$ and the sets of agents in $2^{\cI} \times 2^{\cJ}$. 
The transition of context at step $h$ follows the heterogeneous transition function $\PP_h(C_{h+1} \mid C_h, e_h)$, which only depends on the planner's action $e_h$ and is independent of the matching $(X_h,\tau_h)$. 

We assume that the sequence of two-sided paired sets $\{I_h, J_h\}_{h=1}^H$ is generated independently from other components in this matching market. We also assume the same sequence through all $K$ episodes for the sake of clarity. Note that $I_h$ and $J_h$ can also be handled as part of the context $C_h$ and covered by our current argument with more involved transition dynamics, which is often task specific. Our minor simplification serves to build a generic framework and 
avoid detailed modeling of the agent sets.

\paragraph{Policies and Value Functions.}  
A policy $\pi$ is defined as $\pi = \{\pi_h\}_{h=1}^H$, where for each $s\in \cS$, $\pi_h(\cdot | s)$ is a distribution on $\cA$. 
The policy consists of two parts: the planner's part (i.e. choosing $e\in\Upsilon$ to influence the market context $C$), and the agents' part (i.e. determining the matching-transfer $(X,\tau)$). We use $\Pi$ to denote the set of such policies.

For any policy $\pi$, we define each value function $V_h^\pi (\cdot)$ as 
\begin{align}\label{eq: expected total utility}
    V_h^\pi(s) \coloneqq \EE_{\pi}\Big[ & \sum\nolimits_{l=h}^H r_l(s_l,a_l) \ \Big| \  s_h=s; \ a_l \sim \pi_l(\cdot|s_l),  s_{l+1}\sim\PP_l(\cdot|s_l,a_l), \forall \ h \leq l \leq H \Big].
\end{align} 
Maximizing the accumulated social welfare is equivalent to maximizing $V_1^\pi$ over $\pi\in\Pi$, so the overall performance over $K$ episodes is evaluated through the regret
\begin{align}\label{eq: regret total}
    R(K) \coloneqq  \sum\nolimits_{k=1}^{K} \Big[ \max_{\substack{\pi\in\Pi}} V_1^\pi (s_1) - V_1^{\pi_k} (s_1) \Big],
\end{align} 
where $\pi_k$ denotes the policy in episode $k$. See Appendix \ref{sec: tab_notation} for detailed definitions of our notation.

\section{Method: An Optimisitic Meta Algorithm}

We propose an optimistic RL algorithm for the Markov matching market in this section.
Our proposed algorithm serves as a meta stereotype that can readily incorporate various existing RL methods.

\subsection{Optimistic Estimation of Rewards}\label{sec:reward_estimation}

Note that we do not directly observe the rewards defined in~\eqref{eq: immediate reward} and have no unbiased estimates of them, so we cannot explicitly construct their optimistic estimates.
However, thanks to the nature of the stable matching being a max-weight matching, we show that we can still obtain useful optimistic rewards estimates.

To see this, recall the definition of reward in \eqref{eq: immediate reward}.
We know that there exists some stable matching $(X_h,\tau_h)$ that maximizes $r_h$ and can be obtained by solving a linear program and its dual program~\citep{shapley1971assignment}.
Denote by $\cL\cP(I, J, u,v)$ the following linear program:
\begin{align}\label{eq: linear program definition}
    \begin{aligned}
    \max_{w\in\RR^{|I|\times|J |}} & \sum\nolimits_{(i,j)\in I\times J} w_{i,j}\left[u(i,j) + v(i,j) \right]
    \\ \text{s.t.  } \quad  & \sum\nolimits_{j\in J} w_{i,j} \leq 1, \forall \ i\in I , 
    \\ & \sum\nolimits_{i\in I} w_{i,j}\leq 1, \forall \ j\in J , 
    \\& w_{i,j}\geq 0, \forall \ (i,j)\in I\times J ,
    \end{aligned}
\end{align}
and $\cD\cP(I,J, u,v)$ its dual program:
\begin{align}\label{eq: dual program definition}
    \min_{p: I\cup J \to \RR^{+}} & \sum\nolimits_{a \in I \cup J} p(a)
    \\ \text{s.t.  } \quad  & p(i) + p(j) \geq u(i,j) + v(i,j), \forall(i,j) \in I\times J . \notag
\end{align}
\citet{shapley1971assignment} proved that the stable $(X,\tau)$ correspond to the solution to the linear program \eqref{eq: linear program definition} (for $X$) and its dual program \eqref{eq: dual program definition} (for $\tau$).
It is clear from \eqref{eq: linear program definition} that the optimal value of the objective function is equal to the total social welfare of the stable matching.
Now, suppose we have some optimistic estimates of the utilities, i.e., $\hat u$ and $\hat v$ such that $\hat u(\cdot,\cdot)\geq u(\cdot,\cdot)$ and $\hat v(\cdot,\cdot)\geq v(\cdot,\cdot)$. It is easy to see that when substituting $(\hat u,\hat v)$ into the linear program \eqref{eq: linear program definition}, the resulting optimal value will be an upper bound of the original optimal value (see the proof of Lemma \ref{lem: planner_optimism}).

Based on this observation, let us return to the reward in~\eqref{eq: immediate reward}.
The previous argument implies that as long as we have optimistic estimates of the utilities $u_h$ and $v_h$, we can get optimistic estimates of the reward by solving the max-weight matching based on the optimistic utilities.
Moreover, it is further an upper bound of the following pseudo-reward:
\begin{align}\label{eq: pseudo-reward}
    & \bar{r}_h(C_h, I_h, J_h, e_h) \coloneqq \max_{(X_h, \tau_h) \in \cM_h}  r_h(C_h, I_h, J_h, e_h, X_h, \tau_h),
\end{align}
where $\cM_h \coloneqq \cM(I_h, J_h, u_h, v_h, C_h ,e_h)$ denotes the set of all myopic stable matching on $(I_h, J_h)$ with utility functions $u_h(C_h, e_h, \cdot,\cdot)$ and $u_h(C_h, e_h, \cdot,\cdot)$.
Finally, the optimistic estimates of the utilities can be constructed from noisy observations of agents' utilities via any standard approach in the online learning literature.

The definition of the pseudo-reward in \eqref{eq: pseudo-reward} provides a way to decompose the total regret into the planner's regret and the agents' regret, as will be clear in the next subsection.

\subsection{Decomposition of The Planner and The Agents}

Recall that we require the matching in each step to be stable, which is an additional constraint apart from maximizing the social welfare.
We need to separate these two entangled goals from each other.
Indeed, we will show that the total regret consists of two parts: 1) the suboptimality of the planner's policy \emph{over the entire episode}, and 2) the distance between the proposed matching and the optimal myopic stable matching \emph{at each step}.
We identify the former as the planner's problem and the latter the agents' problem.

\paragraph{The Planner's Problem.} The planner's problem focuses on the transition of the contexts, so we need to partial out the effects from the actual matching.
This has been done in the definition of the pseudo-reward in \eqref{eq: pseudo-reward}, and the corresponding pseudo-value function  $\bar{V}_h^\pi$ for $h \in [H]$ is defined as 
\begin{align}\label{eq: pseudo-value function}
    \bar{V}_h^{\pi}(s) \coloneqq \EE_{\pi} \Big[ & \sum\nolimits_{l=h}^H \bar{r}_l(s_l,e_l) \ \Big| \ s_h=s, \ e_l\sim\pi_l(\cdot|s_l)  s_{l+1}\sim\PP_l(\cdot|s_l,e_l), \forall \ h \leq l \leq H \Big],
\end{align} 
where we slightly abuse the notation $e_l \sim \pi_l(\cdot | s_l)$. Also note that we can write $s_{h+1} \sim \PP(\cdot| s_h, e_h)$ instead of the more general $s_{h+1} \sim \PP(\cdot| s_h, a_h)$ since we condition on $(I_h, J_h)$ and the transition of $C_h$ only depends on the planner's action $e_h$ as $C_{h+1} \sim \PP_h(\cdot | C_h, e_h)$.

Clearly, $\bar V_h^\pi$ is an upper bound of $V_h^\pi$ and does not depend on the actual matching $\{X_h, \tau_h\}_{h \in [H]}$ since it has  been maximized out. 
Now, we specify the planner's problem as trying to maximize the pseudo-value $\bar{V}_1^\pi$, and define the planner's regret given the initial state $s_1$ as
\begin{align}\label{eq: regret planner}
    R^P (K) & \coloneqq \sum\nolimits_{k=1}^K \Big[\max_{\pi} \bar{V}_1^{\pi} (s_1) - \bar{V}_1^{\pi_k} (s_1)  \Big] = \sum\nolimits_{k=1}^K \Big[ \bar{V}_1^{\star} (s_1) - \bar{V}_1^{\pi_k} (s_1)  \Big].
\end{align}
From a control-theoretic perspective, the planner's problem can be viewed as learning an MDP with the same state space $\cS$, the action space reduced to $\Upsilon$, and reward being the value of the myopic max-weight matching at each step. The reward cannot be observed, nor do we have an unbiased estimator. 
From an economic perspective, the planner's problem captures only the market context and not the specific market outcome (i.e. matching).
Now, note that
\begin{align*}
    R(K) &= \sum_{k=1}^K \Big[ \max_{\pi\in\Pi} V_1^\pi(s) - \bar V_1^{\pi_k}(s_1)\Big] \hfill \text{(Planner's regret)}\\
    &\qquad + \sum\nolimits_{k=1}^K \Big[\bar V_1^{\pi_k}(s_1) - V_1^{\pi_k}(s_1)\Big] \hfill \text{(Utility diff.)},
\end{align*}
where the planner's regret has been captured in \eqref{eq: regret planner}, and it remains to control the utility difference on the agents' side.

\paragraph{Agents' Problem.} 
The agents' problem amounts to controlling the suboptimality of each implemented matching, which boils down to SI proposed by \citet{jagadeesan2021learning}.

\begin{definition}[Subset Instability, \citealt{jagadeesan2021learning}]\label{def: subset instatbility}
    Given any agent sets $I, \ J$ and utility functions $u, \ v: I \times J \to \RR$, the Subset Instability $\SI(X, \tau; I,J,u,v)$ of the matching and transfer $(X,\tau)$ is defined as 
    \begin{align*}
        & \max_{I'\times J' \subseteq I\times J} \Big[  \Big( \max_{X'} \sum_{i \in I'}  u(i, X'(i)) + \sum_{j \in J'}  v(X'(j), j) \Big)  - \sum_{i\in I'} \left( u(i,X(j)) + \tau(i) \right) - \sum_{j\in J'} \left( v(X(j),j) + \tau(j) \right) \Big],
    \end{align*}
    where $X(\cdot)$ and $X'(\cdot)$ denotes the matched agent in matching $X$ and $X'$ respectively.
\end{definition}
Subset Instability has several key properties for learning. 
Importantly, it can be shown that given $(I,J,u,v)$, the utility difference between the optimal matching-transfer pair and $(X, \tau)$ is upper bounded by $\SI(X,\tau; I,J,u,v)$.

With a slight abuse of notation, for $s_h = (C_h, I_h, J_h)$ and $a_h = (e_h, X_h, \tau_h)$, we denote by $\SI(s_h, a_h, u_h, v_h)$ the SI of $(X_h, \tau_h)$ given $I_h, J_h$ and $u_h(C_h,e_h, \cdot,\cdot)$ and $u_h(C_h,e_h, \cdot,\cdot)$. 
We define the regret of the agents as 
\begin{align}\label{eq: regret agents}
    R^M (K) \coloneqq \sum\nolimits_{k=1}^K \EE_{\pi_k}\Big[\sum\nolimits_{h=1}^H \SI(s_h,a_h, u_h, v_h)\Big] . 
\end{align}
Moreover, SI itself can be bounded by the sum of optimistic bonuses.
Therefore, quite surprisingly, we can control the planner's regret and the agents' regret at the same time by bounding the bonus sums.
In this way, the total regret can be controlled due to the following proposition.
\begin{proposition}[Proof in Appendix \ref{sec: proof of prop sum of two regrets}]\label{prop: sum of two regrets}
    For $R(K)$, $R^P(K)$ and $R^M(K)$ defined by \eqref{eq: regret total}, \eqref{eq: regret planner}, \eqref{eq: regret agents}, it holds that $R(K) \leq R^P(K) + R^M(K)$.
\end{proposition}

\subsection{A Meta Algorithm}
Now, we are ready to present our meta algorithm as displayed in Algorithm \ref{alg: matching main}.
As is clear from the previous derivations, it suffices to construct optimistic estimates of the utilities, which then induces 1) matchings of agents and 2) optimistic estimates of the value functions.
The latter then enables the optimistic planning for the planner.

\begin{algorithm}[t]
	\caption{Sequential Optimistic Matching (\texttt{SOM})}
	\label{alg: matching main}
	\begin{algorithmic}[1]
    \STATE {\bfseries Require:} $\lambda,\beta_u,\beta_V$
	\STATE {\bfseries Initialize:} $u_h^1\equiv 1, v_h^1 \equiv 1$ and $\cD_h^0 = \emptyset,\forall h \in [H]$ 
	\FOR{episode $k=1,2,\dots,K$}
	\STATE Receive the initial state $s_1^k=(C_1^k, I_1, J_1)$ 
    \STATE Set $\bar{Q}_{H+1}^k \equiv 0$
	\FOR{stage $h = H, H-1, \ldots , 1$}
	\STATE Estimate utilities $(u_h^k, v_h^k) \leftarrow \texttt{UE}(\cD_h^{k-1}, \beta_u,\lambda)$ \alglinelabel{alg:main_UE}
	\STATE Estimate pseudo-reward $\bar{r}_h^k \leftarrow \texttt{RE}(u_h^k, v_h^k, I_h, J_h)$ \alglinelabel{alg:main_RE}
	\STATE Estimate Q-function $\bar{Q}_h^k\leftarrow \texttt{QE}(\cD_h^{k-1}, \bar r_{h}^{k}, \ \bar{Q}_{h+1}^k,\beta_V,\lambda)$
	\ENDFOR 
	\FOR{stage $h = 1 \dots, H$}
	\STATE Planner takes action $ e_h^k \leftarrow \argmax_{e \in \Upsilon} \bar{Q}_h^k (s_h^k , e)$
	\STATE $( X_h^k, \tau_h^k) \leftarrow \texttt{OM} (u_h^k, v_h^k , I_h, J_h , C_h^k , e_h^k)$ \alglinelabel{alg:main_matching} 
	\STATE Implement matching $(X_h^k, \tau_h^k)$ 
	\STATE Observe utilities $u_h^k(i,j), v_h^k(i,j)$ for $(i,j) \in X_h^k$
	\STATE Receive next state $s_{h+1}^k = (C_{h+1}^k, I_{h+1}, J_{h+1})$
	\STATE Update utility dataset $\cD_h^k = \cD_h^{k-1}\cup\{C_h^k, e_h^k\} \cup\{ u_h^k(i,j), v_h^k(i,j)\}_{(i,j) \in X_h^k}$ 
	\ENDFOR
	\ENDFOR
	\end{algorithmic}
\end{algorithm}

In particular, for the estimation part, Algorithm \ref{alg: matching main} first generates the Q-function estimates in a backward fashion. Then, using the estimates, Algorithm \ref{alg: matching main} computes optimistic estimates of utilities by calling the subroutine \texttt{UE}, which then leads to estimates of the pseudo-reward using the subroutine \texttt{RE}.
Next, optimistic estimates of the Q functions are obtained via the subroutine \texttt{QE}.
Next, for the planning part, Algorithm~\ref{alg: matching main}
chooses action in $\Upsilon$ greedily and the matching-transfer pair by calling \texttt{OM}. 
Finally, \texttt{OM} finds a matching-transfer pair which is stable w.r.t. the estimated utility functions.

As discussed in Section \ref{sec:reward_estimation}, the optimistic estimates of the rewards come from solving the linear program in \eqref{eq: linear program definition} and its dual program in \eqref{eq: dual program definition}, which produce the optimal matching given the set of participating agents and utilities.
Therefore, the \texttt{RE} oracle is defined as Algorithm \ref{alg: reward estimation}, and the \texttt{OM} oracle is defined as Algorithm \ref{alg: optimal matching oracle}.

The remaining subroutines (\texttt{UE} and \texttt{QE}) are flexible and can be carefully calibrated for different model assumptions.
The modular nature of Algorithm \ref{alg: matching main} facilitates incorporation of existing RL algorithms.
In particular, we study a special case of linear function approximation in the next section, where we provide explicit oracles for these subroutines, and show that the corresponding algorithm enjoys a sublinear regret.

\section{Case Study: Markov Matching Markets with Linear Features}
In this section, we illustrate the power of our framework under linear function approximation, which is the simplest case of function approximation, yet still a rich enough model. 

\begin{algorithm}[t]
	\caption{Reward Estimation (\texttt{RE})}
	\label{alg: reward estimation}
	\begin{algorithmic}[]
	\STATE {\bfseries Input:} $u, v, I , J$ 
	\STATE {\bfseries Output:} $\hat{r}(C, I, J, e)$ as the solution to the $\cL\cP(I,J, u, v)$, for any $(C, e) \in \cC \times \Upsilon$
	\end{algorithmic}
\end{algorithm}

\begin{algorithm}[t]
	\caption{Optimal Matching (\texttt{OM})}
	\label{alg: optimal matching oracle}
	\begin{algorithmic}[]
	\STATE {\bfseries Input:} $u, v, I , J, C, e$ 
	\STATE Solve for $(X, \tau)$ from the solution of the primal-dual program defined by \eqref{eq: linear program definition} and \eqref{eq: dual program definition}:
	\begin{align*}
	    & \cL\cP(I,J, u(C,e,\cdot,\cdot), v(C,e,I,J))
	    \\ & \cD\cP(I,J, u(C,e,\cdot,\cdot), v(C,e,I,J))
	\end{align*}
    \STATE {\bfseries Output:} $(X, \tau)$ 
	\end{algorithmic}
\end{algorithm}

\subsection{Model Assumptions}
\paragraph{Utility Model.}
We assume there are known feature mappings $\bpsi: \cC \times \Upsilon \to \RR^d$ and $\bphi: \cI \times \cJ \to \RR^d$, such that for any $h \in [H]$ and $(C_h, e_h,i,j)$, the utility functions are
\begin{align*}
    u_h(C_h, e_h, i,j) & = \langle \tvec(\bpsi(C_h, e_h) \bphi(i,j)^\top) , \btheta_h \rangle,
    \\ v_h(C_h, e_h, i,j) & = \langle \tvec(\bpsi(C_h, e_h) \bphi(i,j)^\top), \bgamma_h \rangle.
\end{align*} Here $\btheta_h$ and $\bgamma_h$, $h\in[H]$ are unknown parameters in $\RR^{d^2}$. 
We further define the vectorized feature vector $\bPhi$ as 
\begin{align*}
    \bPhi(C_h, e_h, i,j) \coloneqq \tvec(\bpsi(C_h, e_h)\bphi(i,j)^\top) \in \RR^{d^2}. 
\end{align*}
Then the immediate reward defined in \eqref{eq: immediate reward} can be written as
\begin{align*}
    & r_h(s_h, a_h) = \Big\langle \sum\nolimits_{(i,j)\in X_h} \bPhi(C_h, e_h, i, j), \btheta_h+\bgamma_h \Big\rangle.
\end{align*}

\paragraph{Transition Model.}
Conditioning on the agents' sets $\{I_h, J_h\}_{h=1}^H$, the state transition reduces to that of the contexts. 
We assume a linear transition model \citep{jin2020provably}:
\begin{align}\label{eq: context transition model}
    \PP_h(C_{h+1}|C_h, e_h) \coloneqq \langle \bpsi(C_h, e_h) , \bmu_h(C_{h+1}) \rangle , 
\end{align}
for all $h\in [H]$, where $\bmu_h: \cC \to \RR^d$ is some unknown measure.

Next, we introduce some standard assumptions for matching and linear MDPs \citep{jin2020provably, jagadeesan2021learning}. 

\begin{assumption}\label{assump: matching complexity}
    We assume WLOG that for any $h\in[H]$, $\|u_h(\cdot)\|, \|v_h(\cdot)\| \leq 1$. Assume that for any $(I_h, J_h) \in 2^{\cI} \times 2^{\cJ}$, there exists $W_h>0$ such that for any context $C \in \cC$ and action $e \in \Upsilon$, the max-weight (possibly unstable) matching on $(I_h, J_h)$ with utility functions $u_h(C,e,\cdot,\cdot)$ and $v_h(C, e, \cdot, \cdot)$ has total utility upper bounded by $W_h$. 
\end{assumption} 

The quantities $\{W_h\}_{h=1}^H$ can be viewed as a measure of  complexity of the matching problem, and indeed $\sum_{h=1}^H W_h$ determines the magnitude of $V_1^{\pi^{\star}}$.
Assumption \ref{assump: matching complexity} implies a trivial upper bound that $W_h \leq \min\{|I_h|, |J_h|\}$, but when the max-weight matching involves only a subset of the agents, $W_h$ can be much smaller.
Therefore, we regard $\{W_h\}_{h=1}^H$ as instance-dependent parameters.

\begin{assumption}\label{assump: model and feature}
    We assume WLOG that $\left\| \bpsi(C,e) \right\|_2 \leq 1$ and $\left\| \bphi(i,j) \right\|_2 \leq 1$, implying $\|\bPhi(C,e,i,j)\|_2 \leq 1$, for any $(C,e,i,j)$.
    We assume that for any $h\in[H]$, $\|\btheta_h\|_2\leq d$, $\|\bgamma_h\|_2\leq d$ and$\|\bmu_h(\cdot)\|_2\leq \sqrt{d}$. Moreover, assume that $\max_{h\in[H]} \|\int_\cC f(C)\diff \mu_h(C)  \|_2 \leq \sqrt{d} $ for any function $f: \cC \to \RR$ such that $\sup|f| \leq 1$. 
\end{assumption}

\begin{assumption}\label{assump: data}
    We assume that the observed utilities of matched pairs are the true utilities plus independent 1-subgaussian noise.
\end{assumption}

\subsection{Algorithms}
Based on previous model assumptions, we now present explicit computation oracles for Algorithm~\ref{alg: matching main}.

\paragraph{Utility Estimation.} 
At the beginning of episode $k$, for any $h\in[H]$, denote the available data by $\cD_h^{k-1}$ which consists of $\{u_h^t(i,j)\}_{t\in[k-1]}^{(i,j)\in X_h^t}$ and $\{v_h^t(i,j)\}_{t\in[k-1]}^{(i,j)\in X_h^t}$, where  by default $\cD_h^0 = \phi$.
For the linear case, each $u_h^t(i,j) = \langle\Phi(C_h^t,e_h^t,i,j),\btheta_h\rangle + noise$, and similar for $v_h^t(i,j)$. So we can estimate $\btheta_h$ and $\bgamma_h$ by ridge regression:
\begin{align}\label{eq: utility parameter estimate}
    \btheta_h^k &= (\bSigma_h^k)^{-1} \sum_{t=1}^{k-1} \sum_{(i,j)\in X_h^t} \Phi(C_h^t,e_h^t,i,j) u_h^t(i,j) , \notag
    \\ \bgamma_h^k &= (\bSigma_h^k)^{-1} \sum_{t=1}^{k-1} \sum_{(i,j)\in X_h^t} \Phi(C_h^t,e_h^t,i,j) v_h^t (i,j) , 
    \\ \bSigma_h^k &= \lambda \bI_{d^2} + \sum_{t=1}^{k-1}\sum_{(i,j)\in X_h^t} \Phi(C_h^t,e_h^t,i,j) \Phi(C_h^t,e_h^t,i,j)^\top . \notag
\end{align}
We then add a bonus to ensure optimism in the utility function estimates and we truncate. Denoting $z=(C,e,i,j)$, 
\begin{align}\label{eq: utilitu function estimate}
    \begin{aligned}
    u_h^k(z) &= \left( \langle \Phi(z), \btheta_h^k\rangle + \beta_u \|\Phi(z)\|_{(\bSigma_h^k)^{-1}}\right)_{[-1,1]},\\ 
    v_h^k(z) &=  \left(\langle \Phi(z), \bgamma_h^k\rangle + \beta_u \|\Phi(z)\|_{(\bSigma_h^k)^{-1}}\right)_{[-1,1]}, 
    \end{aligned}
\end{align}
where the truncation ensures estimated utilities in $[-1,1]$.
The \texttt{UE} for the linear case is summarized in Algorithm~\ref{alg: utility estimation}.

\paragraph{Q-function Estimation.}  
By the assumption on the context transition, for any function $f$, the function $\PP_h f$ is linear in features $\bpsi$, which induces the commonly used LSVI-type approach~\citep{jin2020provably}. 
Together with the reward estimates, we can estimate the Q-function using Bellman equation via backward ridge regression. 
For each $(k,h)$, denote the estimate of the Q-function by $\bar{Q}_h^k$ and the value function by $\bar{V}_h^k$. 
Given $\bar{Q}_{h+1}^k$, maximizing over $e\in\Upsilon$ yields $\bar V_{h+1}^k$, then we solve the following ridge regression:
\begin{align*}
    \wb_h^k =  \argmin_{\wb \in \RR^d} \sum_{t=1}^{k-1}&\left[\bar{V}_{h+1}^k(C_{h+1}^{t}) - \bpsi(C_h^t, e_h^t)^\top \wb  \right]^2 + \lambda\|\wb\|_2^2,
\end{align*}
which further yields the estimated expectation of $\bar{V}_{h+1}^k$:
\begin{align}\label{eq:estimate_exp_V}
    \hat{\PP}_h \bar{V}_{h+1}^k (\cdot, \cdot) =  \bpsi(\cdot,\cdot)^{\top} \wb_h^k + \beta_V\| \bpsi(\cdot,\cdot) \|_{ {(\bLambda_h^{k})}^{-1}},
\end{align}
where $\bLambda_h^k =  \sum_{t=1}^{k-1} \bpsi(C_h^t, e_h^t) \bpsi(C_h^t, e_h^t)^\top + \lambda \bI_d$.
Then we estimate $Q_h^k$ by the Bellman equation, as in Algorithm \ref{alg: Q estimation}.

\subsection{Theoretical Results}\label{sec: main results}
\begin{algorithm}[t]
	\caption{Utility Estimation (\texttt{UE})}
	\label{alg: utility estimation}
	\begin{algorithmic}[]
	\STATE {\bfseries Input:} $\cD_h^{k-1}$, $\beta_u$, $\lambda$
    \IF{$\cD_h^{k-1}$ is empty} 
    \STATE\textbf{Output}: $u_h^k \equiv 1$ and $v_h^k \equiv 1$, $\forall h \in [H]$
    \ENDIF
    \STATE Compute $\btheta_h^k$, $\bgamma_h^k$ by \eqref{eq: utility parameter estimate}
    \STATE Estimate utility functions with $u_h^k$ and $v_h^k$ by \eqref{eq: utilitu function estimate}
	\STATE {\bfseries Output:} the functions $u_h^k$ and $v_h^k$
	\end{algorithmic}
\end{algorithm}
\begin{algorithm}[h]
	\caption{Q-function Estimation (\texttt{QE})}
	\label{alg: Q estimation}
	\begin{algorithmic}[]
	\STATE {\bfseries Input:} $\cD_h^{k-1}$, $\bar{r}_h^k$, $\bar{Q}_{h+1}^k$, $\beta_V$, $\lambda$
	\STATE $\bar{V}_{h+1}^k (C,I_h,J_h) = \max_{e \in \Upsilon} \bar{Q}_{h+1}^k(C,e,I_h,J_h)$
	\STATE Compute $\hat\PP_h\bar{V}_{h+1}^k$ by \eqref{eq:estimate_exp_V}
	\STATE $\tilde{Q}_h^k(C, I_h, J_h, e) =  \bar{r}_h^k (C, I_h , J_h, e) + \hat{\PP}_h \bar{V}_{h+1}^k (C, e)$
	\STATE $\bar{Q}_h^k(C, I_h, J_h, e) = (\tilde{Q}_h^k(C, I_h, J_h, e))_{ \left[0, \sum_{l=h}^{H} W_l \right]}$
	\STATE {\bfseries Output:} function $\bar{Q}_h^k$
	\end{algorithmic}
\end{algorithm}
In this section, we present our theoretical results, with the proofs deferred to Appendix \ref{apdx: proof of main}. 
We start with the main theorems on the agents' regret and the planner's regret. 

\begin{theorem}[Agents' Regret]\label{thm: agents regret}
With probability at least $1-2\delta$, the agents' regret can be bounded as $R^M(K) \leq \cO( d^2 (\sum_{h=1}^H \min\{|I_h|, |J_h|\} ) \kappa \sqrt{K})$
    where $\kappa = \log\left( dK \min(|\cI|, |\cJ|)/\delta \right)$.
\end{theorem}

Comparing this result with Theorem 5.3 in \citet{jagadeesan2021learning}, which proves the regret of their \texttt{MatchLinUCB} algorithm designed for the linear utility class under $H=1$, and ignoring the logarithmic term, we see that both have a linear dependence on the cardinality of the agents' set, while we have an extra summation over the $H$ horizon due to the sequential setting. It might seem that their $d$-dependence is $\cO(d)$ while ours is $\cO(d^2)$. But this is because their feature is in $\RR^d$ while our feature $\bPhi \in \RR^{d^2}$. Therefore, the dominant term in Theorem \ref{thm: agents regret} matches that of \citet{jagadeesan2021learning}, and our result can be viewed as an extension.

\begin{theorem}[Planner's Regret]\label{thm: planner regret}
    Under Assumption \ref{assump: matching complexity}, \ref{assump: model and feature}, and $KH>32$, there exists a problem-independent constant $\eta>0$, such that for any $\delta>0$, with the choice of parameters $\lambda = 1$, $\beta_V = \eta d^2 \big( \sum_{h=1}^H W_h\big) \cdot \sqrt{\iota} $ where $\iota = \log ({dKH \min\{|\cI|, |\cJ|\} }/{\delta})$, and $\beta_u$ as given in Lemma \ref{lem: ucb for utility}, then with probability at least $1-\delta$, the planner's regret is bounded by $R^P(K) \leq \cO( \eta d^{5/2} H ( \sum\nolimits_{h=1}^H W_h)\iota\sqrt{K})$.
\end{theorem}

Combining the above two theorems, we have shown that we find the optimal policy for both the planner and agents at sublinear rates, and accordingly for the total regret.
Notably, the regret upper bound only depends on the size of the market through $\{W_h\}_{h=1}^H$, which is instance-dependent.
Due to the sublinear regret, our algorithm can be further adapted to a PAC algorithm~\citep[cf.][]{jin2018q}.
We also remark that our result with LSVI-type estimation can be naturally extended to Eluder dimension~\citep{wang2020reinforcement, ayoub2020model}.

\section{Proof Sketch}\label{sec: proof sketch}
\subsection{Optimistic Utility and Reward Estimates}

The key step in our analysis is to show that the estimated pseudo-reward function $\bar{r}_h^k$ satisfies optimism, i.e., $\bar{r}_h^k \geq \bar{r}_h$, and we need to ensure that $\bar{r}_h^k$ is not too far away from $\bar{r}_h$.
The lemma below justifies the optimism of utility estimates.
\begin{lemma}[UCB for Utility Estimates; proof in~\ref{sec: proof of ucb for utility}]\label{lem: ucb for utility}
For any $0 < \delta<1$, set $\beta_u$ as $\beta_u  = \sqrt{d^2\log[ 2(1 + d^2 K \max_h \min\{|I_h|, |J_h|\})/(\lambda\delta)]} + \sqrt{\lambda} d$.
Then with probability at least $1-\delta$, $ u_h^k$ and $ v_h^k$ in Algorithm~\ref{alg: matching main} satisfy $u_h^k \geq u_h$ and $v_h^k \geq v_h$. 
Furthermore, $|u_h^k(\cdot) - u_h(\cdot)|$ and $|v_h^k(\cdot)-v_h(\cdot)|$ are bounded by $2 \beta_u \left\|\bPhi(\cdot) \right\|_{(\bSigma_h^k)^{-1}}.$
\end{lemma}

Next, we explain why the optimism of utility estimates implies that of reward estimates. 
By Lemma \ref{lem: ucb for utility}, we can write $ u_h^k = u_h + b_{u,h}$ and $ v_h^k = v_h + b_{v,h}$ where $b_{u,h}$ and $b_{v,h}$ are bonus functions satisfying
\begin{align*}
    0 \leq b_{u,h} (C, e , i,j) & \leq 2 \beta_u \left\|\bPhi(C, e, i,j) \right\|_{(\bSigma_h^k)^{-1}}, 
    \\ 0 \leq b_{v,h} (C, e , i,j) & \leq 2 \beta_u \left\|\bPhi(C, e, i,j) \right\|_{(\bSigma_h^k)^{-1}}. 
\end{align*}

\begin{lemma}[Planner's Optimism; proof in \ref{sec: proof of planner optimism}]\label{lem: planner_optimism}
Under the event of Lemma \ref{lem: ucb for utility}, it holds that for any $(C,e) \in \cC\times \Upsilon$,  
\begin{align*}
0 & \leq \bar{r}_h^k(C,e,I_h, J_h) - \bar r_h(C,e,I_h,J_h)  
\leq \sum\nolimits_{(i,j)\in X_h^k} (b_{u,h}(C,e,i,j) + b_{v,h}(C,e,i,j)).
\end{align*}
\end{lemma}

In the sequel, we denote by $\pi_k$ the policy whose market making part is greedy w.r.t.\ $\bar{Q}_h^k$ and whose matching part chooses the max-weight stable matching given $u_h^k$ and $v_h^k$.  

\subsection{Proof Sketch of Theorem \ref{thm: agents regret}}

By definition \eqref{eq: regret agents}, the agents' regret can be interpreted as the expected sum of total SI across all time steps, where the expectation is over the trajectory induced by $\pi_k$ for $k\in[K]$. 

To bound the regret, we relate the expected SI in \eqref{eq: regret agents} with the realized SI via a martingale difference sequence. Specifically, writing $\SI_h = \SI(s_h,a_h,u_h,v_h)$ and $\SI_h^k = \SI(s_h^k,a_h^k,u_h,v_h)$, we define the sum of differences as
\begin{align}\label{eq: agents regret martingale}
    \sum\nolimits_{k=1}^K \Big\{ \EE_{\pi_k} \Big[ \sum\nolimits_{h=1}^H  \SI_h \Big] - \sum\nolimits_{h=1}^H \SI_h^k \Big\} .
\end{align} 
We bound the difference \eqref{eq: agents regret martingale} and the sum of realized SI $\sum_{k=1}^K\sum_{h=1}^H \SI_h^k$ separately, where the former is a sum of martingale difference sequences that concentrates and the latter can be bounded using the following lemma. 

\begin{lemma}[Lemma 5.4 in \citealt{jagadeesan2021learning}; proof in~\ref{sec: proof of SI bound by ucb width}]\label{lem: SI bound by ucb width}
Under the event of Lemma \ref{lem: ucb for utility}, we have
\begin{align*}
    \SI_h^k \leq \sum\nolimits_{(i,j)\in X_h^k} (b_{u,h}(C_h^k,e_h^k,i,j) + b_{v,h}(C_h^k,e_h^k,i,j)) .
\end{align*}
\end{lemma} 

\begin{remark}
Note that each implemented matching induces several utility observations at a time, so bounding the bonus sum for utilities has a resemblance to lazy policy updates in the online learning literature \citep{abbasi2011improved}. It particularly is similar to techniques used in the low switching cost problem in RL \citep{bai2019provably, wang2021provably,gao2021provably}.
This will be clear in the proof of Theorem~\ref{thm: agents regret} in Appendix~\ref{sec: proof of agents regret}.
\end{remark}

\subsection{Proof Sketch of Theorem \ref{thm: planner regret}}
Define the following functions $\delta_h^k$ and terms $\zeta_{k,h}^1, \zeta_{k,h}^2$: 
\begin{align}\label{eq: define terms in regret}
    \delta_h^k(C, e, I, J) & \coloneqq [\bar{r}_h + \PP_h \bar{V}_{h+1}^k  - \bar{Q}_h^k](C,e,I, J)  \,, \notag 
    \\ \zeta_{k,h}^1 & \coloneqq ( \bar{V}_h^k - \bar{V}_h^{\pi_k} )(C_h^k, I_h, J_h)  - ( \bar{Q}_h^k - \bar{Q}_h^{\pi_k} )(C_h^k, e_h^k, I_h, J_h) \,,
    \\ \zeta_{k,h}^2 & \coloneqq \PP_h ( \bar{V}_{h+1}^k - \bar{V}_{h+1}^{\pi_k} )(C_h^k, e_h^k, I_h, J_h) - ( \bar{V}_{h+1}^k - \bar{V}_{h+1}^{\pi_k} )(C_{h+1}^k, I_{h+1}, J_{h+1}). \notag
\end{align} To simplify the notation, in the following, we omit $I_h, J_h$ from the arguments of the functions since we are conditioning on $\{I_h, J_h\}_{h=1}^H$ being fixed.

\begin{lemma}[Regret Decomposition of Planner; proof in~\ref{sec: proof of regret decomp planner}]\label{lem: regret decomp planner}
    The planner's regret defined by \eqref{eq: regret planner} satisfies
    \begin{align}\label{eq:planner_regret_decomp}
        &R^P(K) = \underbrace{\sum_{k=1}^K \sum_{h=1}^H \left( \zeta_{k,h}^1 + \zeta_{k,h}^2 \right)}_{E_1} + \underbrace{\sum_{k=1}^K \sum_{h=1}^H \left[ \EE_{\pi^{\star}}\left[\delta_h^k(C_h,e_h)| C_1^k \right] - \delta_h^k(C_h^k, e_h^k) \right] }_{E_2}
        \\ 
        &\qquad \qquad + \underbrace{\sum_{k=1}^K \sum_{h=1}^H \EE_{\pi^{\star}} \left[ \langle \bar{Q}_h^k(C_h, \cdot), \pi_h^{\star}(\cdot | C_h)- \pi_{k,h}(\cdot | C_h) \rangle_{\Upsilon} \middle| C_1^k \right]}_{E_3}, \notag
    \end{align}
    where the expectation $\EE_{\pi^{\star}}[\cdot | C_1^k]$ is with respect to the trajectory $\{C_h,e_h\}_{h=1}^H$ induced by the policy $\pi^{\star}$ conditioning on $C_1=C_1^k$ and $\langle\cdot,\cdot\rangle_\Upsilon$ means sum over all $e\in\Upsilon$.
\end{lemma}

In decomposition \eqref{eq:planner_regret_decomp}, term $E_1$ is controlled using a standard martingale concentration.
Next, to bound $E_2$, we show that $\delta_h^k\leq 0$ with high probability, which implies that $E_1 \leq \sum_{k=1}^K \sum_{h=1}^H \left| \delta_h^k(C_h^k, e_h^k) \right|$.
Bounding each $|\delta_h^k(C_h^k,e_h^k)|$ by the corresponding optimistic bonus $\|\bpsi(C_h^k,e_h^k)\|_{(\bLambda_h^{k+1})^{-1}}$, we then apply Elliptical Potential lemma to get 
\begin{align*}
    E_1 & \leq 2 \beta_V \sum\nolimits_{h=1}^H \sum\nolimits_{k=1}^K \sqrt{2} \left\| \bpsi(C_h^k,e_h^k) \right\|_{ {(\bLambda_h^{k+1})}^{-1}}  \leq 2 \sqrt{2} \beta_V H \sqrt{Kd \log\left( (K+d)/d\right)}.
\end{align*} 
Finally, for $E_3$, note that by Algorithm \ref{alg: matching main}, the market-making part of the policy $\pi_{k,h}$ is the greedy policy with respect to $\bar{Q}_h^k$. 
Based on this observation, it follows that $\sum_{e\in\Upsilon}  \bar{Q}_h^k(C_h, e)(\pi_h^{\star}(e| C_h)- \pi_{k,h}(e | C_h)) = \sum_{e\in\Upsilon} \bar{Q}_h^k(C_h, e) \pi_h^{\star}(e| C_h) - \max_{e \in \Upsilon} \bar{Q}_h^k(C_h, e) \leq 0$, so $E_3 \leq 0$. 
Combining yields the bound in Theorem~\ref{thm: planner regret}. Full details are presented in Appendix~\ref{sec: complete proof of planner regret}.

\section{Conclusion}
We propose a novel Markov matching market model and a general framework that incorporates max-weight matching and RL algorithms for efficient learning.
We show that our proposed algorithms achieve sublinear regret under proper structural assumptions.

\bibliography{reference}

\begin{thebibliography}{66}
\expandafter\ifx\csname natexlab\endcsname\relax\def\natexlab#1{#1}\fi
\expandafter\ifx\csname url\endcsname\relax
  \def\url#1{\texttt{#1}}\fi
\expandafter\ifx\csname urlprefix\endcsname\relax\def\urlprefix{URL }\fi

\bibitem[{Abbasi-Yadkori et~al.(2011)Abbasi-Yadkori, P{\'a}l and
  Szepesv{\'a}ri}]{abbasi2011improved}
\textsc{Abbasi-Yadkori, Y.}, \textsc{P{\'a}l, D.} and \textsc{Szepesv{\'a}ri,
  C.} (2011).
\newblock Improved algorithms for linear stochastic bandits.
\newblock \textit{Advances in neural information processing systems}
  \textbf{24} 2312--2320.

\bibitem[{Akbarpour et~al.(2014)Akbarpour, Li and
  Gharan}]{akbarpour2014dynamic}
\textsc{Akbarpour, M.}, \textsc{Li, S.} and \textsc{Gharan, S.~O.} (2014).
\newblock Dynamic matching market design.
\newblock \textit{arXiv preprint arXiv:1402.3643} .

\bibitem[{Akbarpour et~al.(2020)Akbarpour, Li and
  Gharan}]{akbarpour2020thickness}
\textsc{Akbarpour, M.}, \textsc{Li, S.} and \textsc{Gharan, S.~O.} (2020).
\newblock Thickness and information in dynamic matching markets.
\newblock \textit{Journal of Political Economy} \textbf{128} 783--815.

\bibitem[{Anderson et~al.(2014)Anderson, Ashlagi, Gamarnik and
  Kanoria}]{anderson2014dynamic}
\textsc{Anderson, R.}, \textsc{Ashlagi, I.}, \textsc{Gamarnik, D.} and
  \textsc{Kanoria, Y.} (2014).
\newblock A dynamic model of barter exchange.
\newblock In \textit{Proceedings of the twenty-sixth annual ACM-SIAM symposium
  on Discrete algorithms}. SIAM.

\bibitem[{Ayoub et~al.(2020)Ayoub, Jia, Szepesvari, Wang and
  Yang}]{ayoub2020model}
\textsc{Ayoub, A.}, \textsc{Jia, Z.}, \textsc{Szepesvari, C.}, \textsc{Wang,
  M.} and \textsc{Yang, L.} (2020).
\newblock Model-based reinforcement learning with value-targeted regression.
\newblock In \textit{International Conference on Machine Learning}. PMLR.

\bibitem[{Azar et~al.(2017)Azar, Osband and Munos}]{azar2017minimax}
\textsc{Azar, M.~G.}, \textsc{Osband, I.} and \textsc{Munos, R.} (2017).
\newblock Minimax regret bounds for reinforcement learning.
\newblock In \textit{International Conference on Machine Learning}. PMLR.

\bibitem[{Azuma(1967)}]{azuma1967weighted}
\textsc{Azuma, K.} (1967).
\newblock Weighted sums of certain dependent random variables.
\newblock \textit{Tohoku Mathematical Journal, Second Series} \textbf{19}
  357--367.

\bibitem[{Baccara et~al.(2020)Baccara, Lee and Yariv}]{baccara2020optimal}
\textsc{Baccara, M.}, \textsc{Lee, S.} and \textsc{Yariv, L.} (2020).
\newblock Optimal dynamic matching.
\newblock \textit{Theoretical Economics} \textbf{15} 1221--1278.

\bibitem[{Bai et~al.(2019)Bai, Xie, Jiang and Wang}]{bai2019provably}
\textsc{Bai, Y.}, \textsc{Xie, T.}, \textsc{Jiang, N.} and \textsc{Wang, Y.-X.}
  (2019).
\newblock Provably efficient q-learning with low switching cost.
\newblock \textit{arXiv preprint arXiv:1905.12849} .

\bibitem[{Basu et~al.(2021)Basu, Sankararaman and
  Sankararaman}]{basu2021beyond}
\textsc{Basu, S.}, \textsc{Sankararaman, K.~A.} and \textsc{Sankararaman, A.}
  (2021).
\newblock Beyond $\log^2 (t)$ regret for decentralized bandits in matching
  markets.
\newblock \textit{arXiv preprint arXiv:2103.07501} .

\bibitem[{Cai et~al.(2020)Cai, Yang, Jin and Wang}]{cai2020provably}
\textsc{Cai, Q.}, \textsc{Yang, Z.}, \textsc{Jin, C.} and \textsc{Wang, Z.}
  (2020).
\newblock Provably efficient exploration in policy optimization.
\newblock In \textit{International Conference on Machine Learning}. PMLR.

\bibitem[{Cen and Shah(2021)}]{cen2021regret}
\textsc{Cen, S.~H.} and \textsc{Shah, D.} (2021).
\newblock Regret, stability, and fairness in matching markets with bandit
  learners.
\newblock \textit{arXiv preprint arXiv:2102.06246} .

\bibitem[{Cesa-Bianchi and Lugosi(2006)}]{cesa2006prediction}
\textsc{Cesa-Bianchi, N.} and \textsc{Lugosi, G.} (2006).
\newblock \textit{Prediction, learning, and games}.
\newblock Cambridge university press.

\bibitem[{Damiano and Lam(2005)}]{damiano2005stability}
\textsc{Damiano, E.} and \textsc{Lam, R.} (2005).
\newblock Stability in dynamic matching markets.
\newblock \textit{Games and Economic Behavior} \textbf{52} 34--53.

\bibitem[{Dann et~al.(2017)Dann, Lattimore and Brunskill}]{dann2017unifying}
\textsc{Dann, C.}, \textsc{Lattimore, T.} and \textsc{Brunskill, E.} (2017).
\newblock Unifying pac and regret: Uniform pac bounds for episodic
  reinforcement learning.
\newblock \textit{arXiv preprint arXiv:1703.07710} .

\bibitem[{Das and Kamenica(2005)}]{das2005two}
\textsc{Das, S.} and \textsc{Kamenica, E.} (2005).
\newblock Two-sided bandits and the dating market.
\newblock In \textit{IJCAI}, vol.~5. Citeseer.

\bibitem[{Doval(2014)}]{doval2014theory}
\textsc{Doval, L.} (2014).
\newblock A theory of stability in dynamic matching markets.
\newblock Tech. rep., Technical report, mimeo.

\bibitem[{Doval(2019)}]{doval2019dynamically}
\textsc{Doval, L.} (2019).
\newblock Dynamically stable matching.
\newblock \textit{Available at SSRN 3411717} .

\bibitem[{Doval and Szentes(2019)}]{doval2019efficiency}
\textsc{Doval, L.} and \textsc{Szentes, B.} (2019).
\newblock On the efficiency of queueing in dynamic matching markets.
\newblock Tech. rep., Working paper, Columbia University.

\bibitem[{Du et~al.(2021)Du, Kakade, Lee, Lovett, Mahajan, Sun and
  Wang}]{du2021bilinear}
\textsc{Du, S.~S.}, \textsc{Kakade, S.~M.}, \textsc{Lee, J.~D.},
  \textsc{Lovett, S.}, \textsc{Mahajan, G.}, \textsc{Sun, W.} and \textsc{Wang,
  R.} (2021).
\newblock Bilinear classes: A structural framework for provable generalization
  in {RL}.
\newblock \textit{arXiv preprint arXiv:2103.10897} .

\bibitem[{Fei et~al.(2021{\natexlab{a}})Fei, Yang, Chen and
  Wang}]{fei2020exponential}
\textsc{Fei, Y.}, \textsc{Yang, Z.}, \textsc{Chen, Y.} and \textsc{Wang, Z.}
  (2021{\natexlab{a}}).
\newblock Exponential bellman equation and improved regret bounds for
  risk-sensitive reinforcement learning.
\newblock In \textit{Advances in Neural Information Processing Systems}.

\bibitem[{Fei et~al.(2020)Fei, Yang, Chen, Wang and Xie}]{fei2020risk}
\textsc{Fei, Y.}, \textsc{Yang, Z.}, \textsc{Chen, Y.}, \textsc{Wang, Z.} and
  \textsc{Xie, Q.} (2020).
\newblock Risk-sensitive reinforcement learning: Near-optimal risk-sample
  tradeoff in regret.
\newblock In \textit{Advances in Neural Information Processing Systems}.

\bibitem[{Fei et~al.(2021{\natexlab{b}})Fei, Yang and Wang}]{fei2021risk}
\textsc{Fei, Y.}, \textsc{Yang, Z.} and \textsc{Wang, Z.} (2021{\natexlab{b}}).
\newblock Risk-sensitive reinforcement learning with function approximation: A
  debiasing approach.
\newblock In \textit{International Conference on Machine Learning}. PMLR.

\bibitem[{Gale and Shapley(1962)}]{gale1962college}
\textsc{Gale, D.} and \textsc{Shapley, L.~S.} (1962).
\newblock College admissions and the stability of marriage.
\newblock \textit{The American Mathematical Monthly} \textbf{69} 9--15.

\bibitem[{Gao et~al.(2021)Gao, Xie, Du and Yang}]{gao2021provably}
\textsc{Gao, M.}, \textsc{Xie, T.}, \textsc{Du, S.~S.} and \textsc{Yang, L.~F.}
  (2021).
\newblock A provably efficient algorithm for linear markov decision process
  with low switching cost.
\newblock \textit{arXiv preprint arXiv:2101.00494} .

\bibitem[{Guo et~al.(2021)Guo, Kandasamy, Gonzalez, Jordan and
  Stoica}]{guo2021online}
\textsc{Guo, W.}, \textsc{Kandasamy, K.}, \textsc{Gonzalez, J.~E.},
  \textsc{Jordan, M.~I.} and \textsc{Stoica, I.} (2021).
\newblock Online learning of competitive equilibria in exchange economies.
\newblock \textit{arXiv preprint arXiv:2106.06616} .

\bibitem[{Gurvich and Ward(2015)}]{gurvich2015dynamic}
\textsc{Gurvich, I.} and \textsc{Ward, A.} (2015).
\newblock On the dynamic control of matching queues.
\newblock \textit{Stochastic Systems} \textbf{4} 479--523.

\bibitem[{Hu and Zhou(2021)}]{hu2021dynamic}
\textsc{Hu, M.} and \textsc{Zhou, Y.} (2021).
\newblock Dynamic type matching.
\newblock \textit{Manufacturing \& Service Operations Management} .

\bibitem[{Jagadeesan et~al.(2021)Jagadeesan, Wei, Wang, Jordan and
  Steinhardt}]{jagadeesan2021learning}
\textsc{Jagadeesan, M.}, \textsc{Wei, A.}, \textsc{Wang, Y.}, \textsc{Jordan,
  M.~I.} and \textsc{Steinhardt, J.} (2021).
\newblock Learning equilibria in matching markets from bandit feedback.
\newblock \textit{arXiv preprint arXiv:2108.08843} .

\bibitem[{Jaksch et~al.(2010)Jaksch, Ortner and Auer}]{jaksch2010near}
\textsc{Jaksch, T.}, \textsc{Ortner, R.} and \textsc{Auer, P.} (2010).
\newblock Near-optimal regret bounds for reinforcement learning.
\newblock \textit{Journal of Machine Learning Research} \textbf{11}.

\bibitem[{Jin et~al.(2018)Jin, Allen-Zhu, Bubeck and Jordan}]{jin2018q}
\textsc{Jin, C.}, \textsc{Allen-Zhu, Z.}, \textsc{Bubeck, S.} and
  \textsc{Jordan, M.~I.} (2018).
\newblock Is q-learning provably efficient?
\newblock \textit{arXiv preprint arXiv:1807.03765} .

\bibitem[{Jin et~al.(2021)Jin, Liu and Miryoosefi}]{jin2021bellman}
\textsc{Jin, C.}, \textsc{Liu, Q.} and \textsc{Miryoosefi, S.} (2021).
\newblock Bellman eluder dimension: New rich classes of rl problems, and
  sample-efficient algorithms.
\newblock \textit{arXiv preprint arXiv:2102.00815} .

\bibitem[{Jin et~al.(2020)Jin, Yang, Wang and Jordan}]{jin2020provably}
\textsc{Jin, C.}, \textsc{Yang, Z.}, \textsc{Wang, Z.} and \textsc{Jordan,
  M.~I.} (2020).
\newblock Provably efficient reinforcement learning with linear function
  approximation.
\newblock In \textit{Conference on Learning Theory}. PMLR.

\bibitem[{Kadam and Kotowski(2018)}]{kadam2018multiperiod}
\textsc{Kadam, S.~V.} and \textsc{Kotowski, M.~H.} (2018).
\newblock Multiperiod matching.
\newblock \textit{International Economic Review} \textbf{59} 1927--1947.

\bibitem[{Kandasamy et~al.(2020)Kandasamy, Gonzalez, Jordan and
  Stoica}]{kandasamy2020mechanism}
\textsc{Kandasamy, K.}, \textsc{Gonzalez, J.~E.}, \textsc{Jordan, M.~I.} and
  \textsc{Stoica, I.} (2020).
\newblock Mechanism design with bandit feedback.
\newblock \textit{arXiv preprint arXiv:2004.08924} .

\bibitem[{Kotowski(2019)}]{kotowski2019perfectly}
\textsc{Kotowski, M.~H.} (2019).
\newblock A perfectly robust approach to multiperiod matching problems .

\bibitem[{Kurino(2020)}]{kurino2020credibility}
\textsc{Kurino, M.} (2020).
\newblock Credibility, efficiency, and stability: A theory of dynamic matching
  markets.
\newblock \textit{The Japanese Economic Review} \textbf{71} 135--165.

\bibitem[{Lattimore and Szepesv{\'a}ri(2020)}]{lattimore2020bandit}
\textsc{Lattimore, T.} and \textsc{Szepesv{\'a}ri, C.} (2020).
\newblock \textit{Bandit Algorithms}.
\newblock Cambridge University Press.

\bibitem[{Lauermann and N{\"o}ldeke(2014)}]{lauermann2014stable}
\textsc{Lauermann, S.} and \textsc{N{\"o}ldeke, G.} (2014).
\newblock Stable marriages and search frictions.
\newblock \textit{Journal of Economic Theory} \textbf{151} 163--195.

\bibitem[{Leshno(2019)}]{leshno2019dynamic}
\textsc{Leshno, J.} (2019).
\newblock Dynamic matching in overloaded waiting lists.
\newblock \textit{Available at SSRN 2967011} .

\bibitem[{Liu(2020)}]{liu2020stability}
\textsc{Liu, C.} (2020).
\newblock Stability in repeated matching markets.
\newblock \textit{arXiv preprint arXiv:2007.03794} .

\bibitem[{Liu et~al.(2020)Liu, Mania and Jordan}]{liu2020competing}
\textsc{Liu, L.~T.}, \textsc{Mania, H.} and \textsc{Jordan, M.} (2020).
\newblock Competing bandits in matching markets.
\newblock In \textit{International Conference on Artificial Intelligence and
  Statistics}. PMLR.

\bibitem[{Liu et~al.(2021)Liu, Ruan, Mania and Jordan}]{liu2021bandit}
\textsc{Liu, L.~T.}, \textsc{Ruan, F.}, \textsc{Mania, H.} and \textsc{Jordan,
  M.~I.} (2021).
\newblock Bandit learning in decentralized matching markets.
\newblock \textit{Journal of Machine Learning Research} \textbf{22} 1--34.

\bibitem[{Loertscher et~al.(2018)Loertscher, Muir and
  Taylor}]{loertscher2018optimal}
\textsc{Loertscher, S.}, \textsc{Muir, E.~V.} and \textsc{Taylor, P.~G.}
  (2018).
\newblock Optimal market thickness and clearing.
\newblock \textit{Unpublished paper, Department of Economics, University of
  Melbourne.[1224]} .

\bibitem[{Mas-Colell et~al.(1995)Mas-Colell, Whinston, Green
  et~al.}]{mas1995microeconomic}
\textsc{Mas-Colell, A.}, \textsc{Whinston, M.~D.}, \textsc{Green, J.~R.}
  \textsc{et~al.} (1995).
\newblock \textit{Microeconomic theory}, vol.~1.
\newblock Oxford university press New York.

\bibitem[{Min et~al.(2021{\natexlab{a}})Min, He, Wang and Gu}]{min2021learning}
\textsc{Min, Y.}, \textsc{He, J.}, \textsc{Wang, T.} and \textsc{Gu, Q.}
  (2021{\natexlab{a}}).
\newblock Learning stochastic shortest path with linear function approximation.
\newblock \textit{arXiv preprint arXiv:2110.12727} .

\bibitem[{Min et~al.(2021{\natexlab{b}})Min, Wang, Zhou and
  Gu}]{min2021variance}
\textsc{Min, Y.}, \textsc{Wang, T.}, \textsc{Zhou, D.} and \textsc{Gu, Q.}
  (2021{\natexlab{b}}).
\newblock Variance-aware off-policy evaluation with linear function
  approximation.
\newblock \textit{Advances in neural information processing systems}
  \textbf{34}.

\bibitem[{Niederle and Yariv(2009)}]{niederle2009decentralized}
\textsc{Niederle, M.} and \textsc{Yariv, L.} (2009).
\newblock Decentralized matching with aligned preferences.
\newblock Tech. rep., National Bureau of Economic Research.

\bibitem[{Osband et~al.(2016)Osband, Van~Roy and
  Wen}]{osband2016generalization}
\textsc{Osband, I.}, \textsc{Van~Roy, B.} and \textsc{Wen, Z.} (2016).
\newblock Generalization and exploration via randomized value functions.
\newblock In \textit{International Conference on Machine Learning}. PMLR.

\bibitem[{{\"O}zkan and Ward(2020)}]{ozkan2020dynamic}
\textsc{{\"O}zkan, E.} and \textsc{Ward, A.~R.} (2020).
\newblock Dynamic matching for real-time ride sharing.
\newblock \textit{Stochastic Systems} \textbf{10} 29--70.

\bibitem[{Qin et~al.(2020)Qin, Tang, Jiao, Zhang, Xu, Zhu and Ye}]{qin2020ride}
\textsc{Qin, Z.}, \textsc{Tang, X.}, \textsc{Jiao, Y.}, \textsc{Zhang, F.},
  \textsc{Xu, Z.}, \textsc{Zhu, H.} and \textsc{Ye, J.} (2020).
\newblock Ride-hailing order dispatching at didi via reinforcement learning.
\newblock \textit{INFORMS Journal on Applied Analytics} \textbf{50} 272--286.

\bibitem[{Rasouli and Jordan(2021)}]{rasouli2021data}
\textsc{Rasouli, M.} and \textsc{Jordan, M.~I.} (2021).
\newblock Data sharing markets.
\newblock \textit{arXiv preprint arXiv:2107.08630} .

\bibitem[{Sankararaman et~al.(2021)Sankararaman, Basu and
  Sankararaman}]{sankararaman2021dominate}
\textsc{Sankararaman, A.}, \textsc{Basu, S.} and \textsc{Sankararaman, K.~A.}
  (2021).
\newblock Dominate or delete: Decentralized competing bandits in serial
  dictatorship.
\newblock In \textit{International Conference on Artificial Intelligence and
  Statistics}. PMLR.

\bibitem[{Satterthwaite and Shneyerov(2007)}]{satterthwaite2007dynamic}
\textsc{Satterthwaite, M.} and \textsc{Shneyerov, A.} (2007).
\newblock Dynamic matching, two-sided incomplete information, and participation
  costs: Existence and convergence to perfect competition.
\newblock \textit{Econometrica} \textbf{75} 155--200.

\bibitem[{Shapley and Shubik(1971)}]{shapley1971assignment}
\textsc{Shapley, L.~S.} and \textsc{Shubik, M.} (1971).
\newblock The assignment game i: The core.
\newblock \textit{International Journal of game theory} \textbf{1} 111--130.

\bibitem[{Strehl et~al.(2006)Strehl, Li, Wiewiora, Langford and
  Littman}]{strehl2006pac}
\textsc{Strehl, A.~L.}, \textsc{Li, L.}, \textsc{Wiewiora, E.},
  \textsc{Langford, J.} and \textsc{Littman, M.~L.} (2006).
\newblock Pac model-free reinforcement learning.
\newblock In \textit{Proceedings of the 23rd international conference on
  Machine learning}.

\bibitem[{Taylor(1995)}]{taylor1995long}
\textsc{Taylor, C.~R.} (1995).
\newblock The long side of the market and the short end of the stick:
  Bargaining power and price formation in buyers', sellers', and balanced
  markets.
\newblock \textit{The Quarterly Journal of Economics} \textbf{110} 837--855.

\bibitem[{{\"U}nver(2010)}]{unver2010dynamic}
\textsc{{\"U}nver, M.~U.} (2010).
\newblock Dynamic kidney exchange.
\newblock \textit{The Review of Economic Studies} \textbf{77} 372--414.

\bibitem[{Vershynin(2010)}]{vershynin2010introduction}
\textsc{Vershynin, R.} (2010).
\newblock Introduction to the non-asymptotic analysis of random matrices.
\newblock \textit{arXiv preprint arXiv:1011.3027} .

\bibitem[{Wang et~al.(2020)Wang, Salakhutdinov and
  Yang}]{wang2020reinforcement}
\textsc{Wang, R.}, \textsc{Salakhutdinov, R.} and \textsc{Yang, L.~F.} (2020).
\newblock Reinforcement learning with general value function approximation:
  Provably efficient approach via bounded eluder dimension.
\newblock \textit{arXiv preprint arXiv:2005.10804} .

\bibitem[{Wang et~al.(2021)Wang, Zhou and Gu}]{wang2021provably}
\textsc{Wang, T.}, \textsc{Zhou, D.} and \textsc{Gu, Q.} (2021).
\newblock Provably efficient reinforcement learning with linear function
  approximation under adaptivity constraints.
\newblock \textit{arXiv preprint arXiv:2101.02195} .

\bibitem[{Yang and Wang(2019)}]{yang2019sample}
\textsc{Yang, L.} and \textsc{Wang, M.} (2019).
\newblock Sample-optimal parametric q-learning using linearly additive
  features.
\newblock In \textit{International Conference on Machine Learning}. PMLR.

\bibitem[{Yang et~al.(2020)Yang, Jin, Wang, Wang and Jordan}]{yang2020function}
\textsc{Yang, Z.}, \textsc{Jin, C.}, \textsc{Wang, Z.}, \textsc{Wang, M.} and
  \textsc{Jordan, M.~I.} (2020).
\newblock On function approximation in reinforcement learning: Optimism in the
  face of large state spaces.
\newblock \textit{arXiv preprint arXiv:2011.04622} .

\bibitem[{Zanette et~al.(2020)Zanette, Lazaric, Kochenderfer and
  Brunskill}]{zanette2020learning}
\textsc{Zanette, A.}, \textsc{Lazaric, A.}, \textsc{Kochenderfer, M.} and
  \textsc{Brunskill, E.} (2020).
\newblock Learning near optimal policies with low inherent bellman error.
\newblock In \textit{International Conference on Machine Learning}. PMLR.

\bibitem[{Zenios(1999)}]{zenios1999modeling}
\textsc{Zenios, S.~A.} (1999).
\newblock Modeling the transplant waiting list: A queueing model with reneging.
\newblock \textit{Queueing systems} \textbf{31} 239--251.

\bibitem[{Zhou et~al.(2021)Zhou, Gu and Szepesvari}]{zhou2021nearly}
\textsc{Zhou, D.}, \textsc{Gu, Q.} and \textsc{Szepesvari, C.} (2021).
\newblock Nearly minimax optimal reinforcement learning for linear mixture
  markov decision processes.
\newblock In \textit{Conference on Learning Theory}. PMLR.

\end{thebibliography}
\bibliographystyle{ims}

\appendix

\section{Clarification of Notation} \label{sec: tab_notation}

In this section, we provide a comprehensive clarification on the use of notation in this paper. 

Throughout the paper, we use $\cO(\cdot)$ to hide problem-independent constants and use $(\cdot)_{[a,b]}$ to denote the truncation into the range $[a,b]$. 
$\cI$ and $\cJ$ denote the sets of all side-1 and side-2 agents respectively. 
Further, we use $I \in 2^{\cI}$ and $J \in 2^{\cJ}$ to denote any set of participating agents, which are the subsets of $\cI$ and $\cJ$. 

For any given stage $h\in[H]$, we use $s_h$ and $(C_h, I_h, J_h)$ interchangeably when describing any state $s_h \in \cS$ where $\cS= \cC \times 2^{\cI} \times 2^{\cJ}$. Analogously, we also use $a_h$ and $(e_h, X_h, \tau_h)$ interchangeably for the action $a_h \in \cA$ where $\cA= \Upsilon \times \cX \times \cT$ 

We use $\pi = \{\pi_h\}_{h=1}^H$ to denote a policy, where each $\pi_h$ is defined to be a mapping from $\cS$ to a distribution $\Delta_\cA$ on $\cA$. Therefore, for any $h\in[H]$ and $s_h \in \cS$, $\pi_h(\cdot| s_h)$ denotes a probability distribution on $\cA$. Note that because $\cA = \Upsilon \times \cX \times \cT$, the policy $\pi$ is a joint policy. We may slightly abuse the notation in the paper and refer to $\pi$ as the policy restricted on $\Upsilon$ only, whenever it is clear from the context. In such case, we refer to $\pi_h(\cdot| s_h)$ as a distribution on $\Upsilon$ only. 

We also present the following table of notations. The $\pi$ in the superscript can be replaced by $\pi_k$ or $\pi_\star$, where the former refers to the policy in episode $k$, and the latter refers to the optimal policy.

\begin{table}[!ht]
\caption{Notation}
\centering
\renewcommand*{\arraystretch}{1.5}
\begin{tabular}{ >{\centering\arraybackslash}m{2cm} | >{\centering\arraybackslash}m{12cm} } 
\hline\hline
Notation & Meaning \\ 
\hline

$\cC, \cI, \cJ$ & set of contexts, side-1 agents, side-2 agents\\ 

$\Upsilon, \cX, \cT$ & set of planner's actions, all matchings and transfers over all possible subsets of $\cI \times \cJ$\\ 

\hline

$r_h, \bar{r}_h$ & reward, pseudo-reward functions \\ 

$V_h^{\pi}, Q_h^{\pi}, V^{\star}_h$ & value, Q functions under $\pi$, optimal value functions w.r.t. the transition functions $\{\PP_h\}_{h=1}^H$ and reward functions $\{r_h\}_{h=1}^H$ \\ 

$\bar{V}_h^{\pi}, \bar{Q}_h^{\pi}, \bar{V}_h^{\star}$ & pseudo-value, pseudo-Q functions under $\pi$, optimal pseudo-value functions w.r.t. the transition functions $\{\PP_h\}_{h=1}^H$ and reward functions $\{\bar{r}_h\}_{h=1}^H$ \\ 

\hline
$\bar{V}_h^k, \bar{Q}_h^k$ & estimated value, Q functions for stage $h$ in episode $k$ in Algorithm \ref{alg: matching main}\\

\hline
$\pi_k$ & the policy followed by Algorithm \ref{alg: matching main} in episode $k$, where $\pi_k = \{\pi_{k,h}\}_{h=1}^H$\\

\hline
$\pi_{k,h}$ & the policy followed by Algorithm \ref{alg: matching main} at stage $h$ in episode $k$ \\

\hline \hline
\end{tabular}
\label{tab:notation}
\end{table}

\section{Supplementary Information on Matching and Stability}\label{sec: all details about matching stability and SI}

In this section, we review some basics on the matching problem. 
We first introduce the classic problem of (static) matching with transfers and the notion of stability. 
We then recap the primal-dual formulation that provides an efficient way to solve a stable matching \citep{shapley1971assignment}. 
Finally, we give more details about Subset Instability and its properties.

\subsection{Matching with Transferable Utilities}

This section is a supplementary to Section \ref{sec:pre_MMM} in the main text. 
We introduce the two-sided static matching with transferable utilities. 

Denote the sets of participating agents by $I$ and $J$ for two sides respectively. 
A matching $X\subseteq I\times J$ is a set of pairs of agents, and $(i,j)\in X$ means $i\in I$ is matched to $j \in J$. 
Each agent is matched at most once.
We denote by $X(i) = j$ and $X(j) = i$ for any matched pair $(i,j) \in X$, while for any unmatched agent $a \in I \cup J$, we write $X(a) = a$.

Matched agents receive utilities, denoted by $u: I \times J \to \RR$ for agents in $I$ and $v: I \times J \to \RR$ for agents in $J$. 
Specifically, if $(i,j) \in X$, then agent $i$ receives an utility $u(i,j)$ and agent $j$ receives an utility $v(i,j)$. Remaining unmatched agents receive zero utility.
With a slight abuse of notation, we overwrite $u(i,i) = 0$ and $v(j,j) = 0$ for any $i\in I$ and $j \in J$.

In addition, there are utility transfers between (and only between) agents.
We denote the transfer function by $\tau: I \times J \to \RR$ such that for any agent $a \in I \cup J$, $\tau(a)$ is the transfer received by agent $a$. 
Since the transfers are within agents, we have 
\begin{align*}
    \sum_{a \in I \cup J} \tau(a) = 0.
\end{align*}

We denote the market outcome by $(X,\tau)$, under which the net utility received by an agent $i\in I$ is $u(i,j) + \tau(i)$ if $(i,j) \in X$, and similarly for agents in $J$. 

The notion of \emph{stable matching} is as follows.
\begin{definition}[Stable matching]\label{def: stable matching}
    A matching-transfer pair $(X,\tau)$ on $I, J$ is stable if:
    \begin{enumerate}
        \item The net utility of of any agent is non-negative, i.e. 
        \begin{align*}
            & u(i,X(i)) + \tau(i) \geq 0,
            \\ & v(X(j), j) + \tau(j) \geq 0,
        \end{align*}
        for all $i \in I$ and $j \in J$.
        
        \item There are no blocking pairs, i.e. 
        \begin{align*}
            \left[ u(i, X(i)) + \tau(i) \right] + \left[ v(X(j), j) + \tau(j) \right] \geq u(i,j) + v(i,j),
        \end{align*}
        for all pairs $(i,j) \in I \times J$. 
    \end{enumerate}
\end{definition}

Stable matching implies that no matched agents would rather be unmatched and no pair of agents can find a transfer between themselves so that both would rather match with each other than follow $(X,\tau)$. 
The following proposition provides a fundamental and important max-weight interpretation for stable matchings.  

\begin{proposition}[\citealt{shapley1971assignment}]\label{prop: stable is max} For the matching with transfer problem, if $(X,\tau)$ is a stable matching under Definition \ref{def: stable matching}, then $X$ must be the max-weight matching, i.e., 
\begin{align*}
    X  = \arg\max_{X'} \sum_{i\in I, j\in J} u(i,X'(i)) + v(X'(j), j)
\end{align*}where the maximum is over all matchings on $I \times J$.
\end{proposition}

Therefore, by Proposition \ref{prop: stable is max}, to maximize the total social welfare (i.e. sum of utilities), it suffices to find a stable matching. 
But how?
This is answered in the next subsection.

\subsection{The Linear Program and Dual Program}\label{sec: detail on linear program and oracle}

In this subsection, we explain how to find a stable matching $(X,\tau)$ given input $I,J,u,v$, which gives rise to the algorithm \texttt{OM} (i.e. Algorithm \ref{alg: optimal matching oracle}) in the main text. 

\citet{shapley1971assignment} showed that, assuming the utility functions are known, the stable $(X,\tau)$ can be found by solving the following linear program and its dual program (recapped from Section \ref{sec:reward_estimation} in the main text):
\begin{align}\label{eq: def of LP in apdx}
    \begin{aligned}
    \cL\cP(I, J, u,v): \qquad \qquad \max_{w\in\RR^{|I|\times|J |}} & \sum_{(i,j)\in I\times J} w_{i,j}\left[u(i,j) + v(i,j) \right] \qquad\qquad\qquad
    \\ \text{s.t.  } \quad  & \sum_{j\in J_h} w_{i,j} \leq 1, \forall \ i\in I , 
    \\ & \sum_{i\in I_h} w_{i,j}\leq 1, \forall \ j\in J , 
    \\& w_{i,j}\geq 0, \forall \ (i,j)\in I\times J ,
    \end{aligned}
\end{align}and its dual program:
\begin{align}\label{eq: def of DP in apdx}
    \cD\cP(I,J, u,v): \qquad \qquad \min_{p: I\cup J \to \RR^{+}} & \sum_{a \in I \cup J} p(a)
    \\ \text{s.t.  } \quad  & p(i) + p(j) \geq u(i,j) + v(i,j), \ \forall(i,j) \in I\times J . \notag
\end{align}

Denote the solution pair to the primal-dual problems by $(w,p)$. \citet{shapley1971assignment} proved that $(w,p)$ leads to a max-weight stable matching-transfer pair $(X,\tau)$. 
Specifically, it is proved that the vector $w$ must have integer entries, i.e., $w_{i,j} = 0$ or $1$, which naturally induces a matching $X$ such that $(i,j)\in X$ if and only if $w_{i,j}=1$.
Correspondingly, the transfers are $\tau(i) = p(i) - u(i, 
X(i))$ for $i \in I$, and similarly for $j\in J$. 

The above procedure constitutes the subroutine oracle \texttt{OM} as displayed in Algorithm \ref{alg: optimal matching oracle}.
It takes as input the sets of participating agents and estimated utility functions, then outputs the stable matching $(X,\tau)$ by solving the primal-dual linear program described above. 
Note that the matching is only stable with respect to the estimated utility functions. 

\subsection{Details about Subset Instability}\label{sec: details of SI}
Next, we review the notion of Subset Instability and its properties. 
We refer the interested reader to \citep{jagadeesan2021learning} for the full details.

Given a matching-pair $(X,\tau)$, define its utility difference as
\begin{align}\label{eq: utility difference}
    \Bigg[ \max_{X'} \sum_{i \in I, j \in J} u(i, X'(i)) + v(X'(j), j) \Bigg] - \Bigg[  \sum_{i \in I, j \in J} u(i, X(i)) + v(X(j), j) \Bigg].
\end{align}
Recall the definition of Subset Instability:
\begin{repdefinition}{def: subset instatbility}[Subset Instability, \citealt{jagadeesan2021learning}]
    Given any agent sets $I, \ J$ and utility functions $u, \ v: I \times J \to \RR$, the Subset Instability $\SI(X, \tau; I,J,u,v)$ of the matching and transfer $(X,\tau)$ is defined as 
    \begin{align*}
        \max_{I'\times J' \subseteq I\times J} \Big[  \Big( \max_{X'} \sum_{i \in I'}  u(i, X'(i)) + \sum_{j \in J'}  v(X'(j), j) \Big) - \Big( \sum_{i\in I'} \left( u(i,X(j)) + \tau(i) \right)  \Big) - \Big( \sum_{j\in J'} \left( v(X(j),j) + \tau(j) \right)  \Big) \Big]
    \end{align*}
    where $X(\cdot)$ and $X'(\cdot)$ denotes the matched agent in matching $X$ and $X'$ respectively.
\end{repdefinition}

By definition, Subset Instability indicates whether there exists any subset $I' \times J'$ of agents who can achieve a higher total utility by taking some alternative matching $X'$ among themselves other than the current matching-transfer pair $(X,\tau)$. 
Thus, Subset Instability is an upper bound of the total utility difference, and quantifies the distance from a proposed matching to the optimal stable matching, as summarized in the following proposition.

\begin{proposition}[Proposition 4.4 in \citealt{jagadeesan2021learning}]\label{prop: properties of SI}
The following holds for Subset Instability
\begin{enumerate}
    \item Subset Instability is always nonnegative and is zero if and only if $(X,\tau)$ is stable matching.
    \item Subset Instability is Lipschitz continuous with respect to the $\ell_\infty$ norm of the utility functions.
    \begin{align*}
        & \left| \SI(X,\tau; I,J,u,v) - \SI(X,\tau; I,J, \tilde{u}, \tilde{v}) \right|  \leq 2 \Big( \sum_{i \in I} \left\| u(i,\cdot) - \tilde{u}(i,\cdot) \right\|_{\infty}  - \sum_{j \in J} \left\| v(\cdot,j) - \tilde{v}(\cdot ,j) \right\|_{\infty}  \Big).
    \end{align*}
    \item Subset Instability is always at least the utility difference \eqref{eq: utility difference}.
\end{enumerate}
\end{proposition}

In our problem, this allows us to bound the total regret of the agents by the sum of Subset Instability of matchings across all episodes, as reflected in Proposition \ref{prop: sum of two regrets}.

\section{Proof of the Main Theory}\label{apdx: proof of main}

\subsection{Proof of Proposition \ref{prop: sum of two regrets}}\label{sec: proof of prop sum of two regrets}
\begin{proof}[Proof of Proposition \ref{prop: sum of two regrets}]
By \eqref{eq: regret total}, the total regret can be written as 
    \begin{align*}
        R(K) & = \sum_{k=1}^K \left[ V_1^{\star}(s) - V_1^{\pi_k} (s) \right] = \sum_{k=1}^K \left[ V_1^{\star}(s) - \bar{V}_1^{\pi_k} (s) \right] + \sum_{k=1}^K \left[\bar{V}_1^{\pi_k} (s) - V_1^{\pi_k} (s) \right],
    \end{align*}where $\bar{V}_1^{\pi_k}$ is the pseudo-value function defined by \eqref{eq: pseudo-value function} corresponding to the pseudo-reward $\bar{r}_h$ (defined by \eqref{eq: pseudo-reward}) and induced by the policy $\pi_k$. Note that here we only care about the $\Upsilon$ part of $\pi_k$ since matching-transfer has been maximized out by the definition of $\bar{r}_h$. 
    
    Now for any policy $\pi \in \Pi$ where $\pi = \{\pi_h\}_{h \in [H]}$, there exists a counterpart $\pi' =\{\pi'_h\}_{h \in [H]}$, such that $\pi_h(C| s) = \pi'_h(C| s)$ for all $C \in \Upsilon$ and $s \in \cS$, whereas for the matching part $\pi'$ always chooses the stable matching w.r.t. the true and unknown utility functions $u_h(C_h,e_h,\cdot,\cdot)$ and $v_h(C_h,e_h,\cdot,\cdot)$. Since $\pi'$ can be viewed as a function of $\pi$, we write $\pi' = \pi'(\pi)$.
    Since the matching does not affect the context transition by assumption, and the stable matching maximizes total utility by Proposition \ref{prop: stable is max}, we then have that 
    \begin{align*}
        V_1^{\star} = \max_{\pi \in \Pi} V_1^{\pi} & = \max_{\pi \in \Pi} \EE_{\pi}\bigg[  \sum_{h=1}^H r_h(s_h,a_h) \ \bigg| \  s_1=s; \ a_h \sim \pi_h(\cdot|s_h), s_{h+1}\sim\PP_h(\cdot|s_h,a_h), \forall \ h \in [H] \bigg]
        \\ & = \max_{\pi'(\pi): \pi \in \Pi } \EE_{\pi'}\bigg[  \sum_{h=1}^H r_h(s_h,a_h) \ \bigg| \  s_1=s; \ a_h \sim \pi'_h(\cdot|s_h), s_{h+1}\sim\PP_h(\cdot|s_h,a_h), \forall \ h \in [H] \bigg]
        \\ & = \max_{\pi'(\pi): \pi \in \Pi} \EE_{\pi'}\bigg[  \sum_{h=1}^H \bar{r}_h(s_h,e_h) \ \bigg| \  s_1=s; \ e_h \sim \pi'_h(\cdot|s_h), s_{h+1}\sim\PP_h(\cdot|s_h,e_h), \forall \ h \in [H] \bigg] = \bar{V}_1^{\star},
    \end{align*}where the fourth step is by definition of $\bar{r}_h$. It follows that 
    \begin{align}\label{eq: proof of R_k bound eq 2}
        R(K) & = \sum_{k=1}^K \left[ \bar{V}_1^{\star}(s) - \bar{V}_1^{\pi_k} (s) \right] + \sum_{k=1}^K \left[\bar{V}_1^{\pi_k} (s) - V_1^{\pi_k} (s) \right] = R^P(K) + \sum_{k=1}^K \left[\bar{V}_1^{\pi_k} (s) - V_1^{\pi_k} (s) \right]. 
    \end{align}
    The second term in the R.H.S. can be written as
    \begin{align*}
        & \sum_{k=1}^K \left[\bar{V}_1^{\pi_k} (s) - V_1^{\pi_k} (s) \right]  = \EE_{\pi_k}\bigg[  \sum_{h=1}^H \bar{r}_h(s_h,e_h) - r_h(s_h, a_h)\ \bigg| \  s_1=s; \ a_h \sim \pi_{k,h}(\cdot|s_h), s_{h+1}\sim\PP_h(\cdot|s_h,a_h), \forall \ h \in [H] \bigg].
    \end{align*}
    Note that $\bar{r}_h(s_h,e_h) - r_h(s_h, a_h)$ is exactly the utility difference defined by \eqref{eq: utility difference}. Since Subset Instability is at least the utility difference by Proposition \ref{prop: properties of SI}, it follows that 
    \begin{align*}
        \sum_{k=1}^K \left[\bar{V}_1^{\pi_k} (s) - V_1^{\pi_k} (s) \right] &\leq \sum_{k=1}^K \EE_{\pi_k}\bigg[  \sum_{h=1}^H \SI(s_h,a_h, u_h,v_h)\ \bigg| \  s_1=s; \ a_h \sim \pi_{k,h}(\cdot|s_h), s_{h+1}\sim\PP_h(\cdot|s_h,a_h), \forall \ h \in [H] \bigg]
        \\ 
        & = R^M(K),
    \end{align*}where the last step is by the definition of $R^M(K)$ in \eqref{eq: regret agents}. Plugging into \eqref{eq: proof of R_k bound eq 2}, we get 
    \begin{align*}
        R(K) \leq R^P(K) + R^M(K). 
    \end{align*}
    This completes the proof.
\end{proof}

\subsection{Bounds for Regression Estimators}

We first bound the norm of the regression estimators in the algorithm.

\begin{lemma}\label{lem: bound of regression estimator of PV}
    The regression estimators $\wb_h^k$ in Algorithm \ref{alg: matching main} satisfy
    \begin{align*}
        \left\| \wb_h^k \right\|_2 & \leq \Big( \sum_{l=h}^{H} W_l \Big) \cdot \sqrt{\frac{dk}{\lambda}}, \quad\left\|\btheta_h^k \right\|_2 \leq \sqrt{\frac{d^2 k \cdot \min\{|\cI|, |\cJ| \}}{\lambda}}, \quad \left\|\bgamma_h^k \right\|_2 \leq \sqrt{\frac{d^2 k \cdot \min\{|\cI|, |\cJ| \}}{\lambda}}.
    \end{align*}
\end{lemma}

\begin{proof}[Proof of Lemma \ref{lem: bound of regression estimator of PV}]
    Consider $\wb_h^k$ for arbitrary $h,k$. For any vector $\mathbf{v} \in \RR^d$,
    \begin{align*}
        \left| \mathbf{v}^\top \wb_h^k\right| & = \Big| \left(\bLambda_h^k\right)^{-1} \sum_{t=1}^{k-1}\bpsi(C_h^t, e_h^t) \bar{V}_{h+1}^k(C_{h+1}^t) \Big| \leq \sum_{t=1}^{k-1} \left| \mathbf{v}^\top \left( \bLambda_h^k \right)^{-1} \bpsi(C_h^t, e_h^t) \right| \cdot \sum_{l=h}^{H} W_l
        \\ &\leq \sum_{l=h}^{H} W_l \cdot \sqrt{\Big[ \sum_{t=1}^{k-1} \mathbf{v}^\top \left( \bLambda_h^k\right)^{-1} \mathbf{v}\Big]\cdot \Big[ \sum_{t=1}^{k-1}\bpsi(C_h^t,e_h^t)^\top \left( \bLambda_h^k\right)^{-1}\bpsi(C_h^t,e_h^t)\Big]}
        \\ & \leq \Big( \sum_{l=h}^{H} W_l \Big) \cdot \sqrt{dk/\lambda} \cdot \| \mathbf{v}\|_2,
    \end{align*}where the first inequality holds because of the truncation of $\bar{Q}_h^k$ and $\bar{V}_h^k(C,I,J) = \max_{e} \bar{Q}_h^k(C,e,I,J)$, the second inequality is by the Cauchy-Schwarz inequality, and the last inequality is by Lemma \ref{lem: matrix inequality gram sum}. Since above holds for any $\mathbf{v}$, we conclude that
    \begin{align*}
        \left\| \wb_h^k \right\|_2 & = \max_{\mathbf{v}\in \RR^d: \|\mathbf{v}\|_2=1} \left| \mathbf{v}^\top\wb_h^k\right| \leq \Big( \sum_{l=h}^{H} W_l\Big) \cdot \sqrt{dk/\lambda}. 
    \end{align*}
    
    For $\btheta_h^k$, let $\mathbf{v} \in \RR^{d\times d}$. Then by the same analysis we have
    \begin{align*}
        \left| \mathbf{v}^\top \btheta_h^k \right| & = \Big| \mathbf{v}^\top \left( \bSigma_h^k \right)^{-1} \sum_{t=1}^{k-1} \sum_{(i,j)\in X_h^t}\bPhi(C_h^t, e_h^t, i,j) u_h^t (i,j)\Big| \leq \sum_{t=1}^{k-1} \sum_{(i,j)\in X_h^t} \left| \mathbf{v}^\top \left( \bSigma_h^k\right)^{-1} \bPhi(C_h^t, e_h^t, i,j) \right| \cdot 1
        \\ & \leq \sqrt{\Big[ \sum_{t=1}^{k-1} \sum_{(i,j)\in X_h^t}  \mathbf{v}^\top \left( \bSigma_h^k\right)^{-1} \mathbf{v} \Big] \cdot \Big[ \sum_{t=1}^{k-1} \sum_{(i,j)\in X_h^t} \bPhi(C_h^t, e_h^t, i,j)^\top \left( \bSigma_h^k\right)^{-1} \bPhi(C_h^t, e_h^t, i,j) \Big] \cdot } 
        \\ & \leq \|\mathbf{v}\|_2 \cdot \sqrt{k \min\{|\cI|, |\cJ|\} \cdot d^2/\lambda},
    \end{align*}which implies
    \begin{align*}
        \left\|\btheta_h^k \right\|_2 \leq \sqrt{\frac{d^2 k \cdot \min\{|\cI|, |\cJ| \}}{\lambda}}.
    \end{align*}The same holds for $\bgamma_h^k$. 
\end{proof}

\subsection{Proof of Theorem \ref{thm: agents regret}}\label{sec: proof of agents regret}

We present the complete proof of Theorem \ref{thm: agents regret}.

\begin{proof}[Proof of Theorem \ref{thm: agents regret}]
    Recall from Proposition \ref{prop: properties of SI} that Subset Instability is at least the utility difference. We thus have
    \begin{align*}
    R^M (K) &= \sum_{k=1}^K \left[ \max_{\substack{\pi\in \Pi }} V_1^{\pi} (s_1) - V_1^{\pi_k} (s_1) \right] \leq\sum_{k=1}^K \EE_{\pi_k}\left[\sum_{h=1}^H \SI(s_h,a_h, u_h, v_h)\right] \\ 
    &= \sum_{k=1}^K \bigg\{ \underbrace{ \EE_{\pi_k} \left[ \sum_{h=1}^H  \SI (s_h,a_h,u_h,v_h) \right]}_{R^M_k} - \underbrace{\sum_{h=1}^H \SI(s_h^k, a_h^k, u_h, v_h)}_{\tilde{R}_k^M}  \bigg\} + \sum_{k=1}^K \sum_{h=1}^H \SI(s_h^k, a_h^k, u_h, v_h) .
    \end{align*} 
    
    For each $k\in[K]$, let $\cF_{k-1}$ denote all the history until the beginning of Episode $k$. Then $\pi_k \sim \cF_{k-1}$ and $R_k^M \sim \cF_{k-1}$. Furthermore, we have $R_k^M  = \EE[ \tilde{R}_k^M \mid \cF_{k-1}]$, i.e., $R_k^M$ is the conditional expectation of the realized quantity $\tilde{R}_k^M$. We can view \eqref{eq: agents regret martingale} as the sum of a martingale difference sequence. From Definition \ref{def: subset instatbility}, it holds almost surely that 
    \begin{align*}
        \left| \EE_{\pi_k} \left[\SI(s_h,a_h,u_h,v_h) \right] - \SI(s_h^k, a_h^k, u_h, v_h) \right| \leq W_h.
    \end{align*}

    By Lemma \ref{lem: azuma hoeffding vanilla}, we have that, for any $0<\delta<1$, with probability at least $1-\delta$, 
    \begin{align}\label{eq: proof agents regret martingale diff}
    \sum_{l=1}^K \left( R_k^M - \tilde{R}_k^M \right) \leq \Big(\sum_{h=1}^H W_h \Big) \sqrt{2K\log\left( \frac{2}{\delta} \right)} . 
    \end{align}
    To bound the second term $\sum_{k=1}^K \tilde{R}_k^M$, note that by Lemma \ref{lem: ucb for utility} and Lemma \ref{lem: SI bound by ucb width}, we have 
    \begin{align}\label{eq: proof agents regret eq 2}
        \sum_{k=1}^K \tilde{R}_k^M & \leq \sum_{k=1}^K \sum_{h=1}^H \sum_{(i,j)\in X_h^k} \left( 4 \beta_u \left\| \bPhi(C_h^k, e_h^k, i,j) \right\|_{(\bSigma_h^k)^{-1}}\right)  = 4 \beta_u \sum_{h=1}^H \Big(  \sum_{k=1}^K \sum_{(i,j)\in X_h^k} \left\| \bPhi(C_h^k, e_h^k, i,j) \right\|_{(\bSigma_h^k)^{-1}} \Big). 
    \end{align}Note that by definition, we have 
    \begin{align*}
        \bSigma_h^{k+1} = \bSigma_h^k + \sum_{(i,j)\in X_h^k} \bPhi(C_h^k, e_h^k, i,j)\bPhi(C_h^k, e_h^k, i,j)^\top.
    \end{align*}
    By picking $\lambda = 1$ and the assumption that $\|\bPhi\| \leq 1$, we have
    \begin{align*}
        \sum_{(i,j)\in X_h^k} \bPhi(C_h^k, e_h^k, i,j)\bPhi(C_h^k, e_h^k, i,j)^\top \preccurlyeq \min(|I_h|, |J_h|) \cdot \bI_{d^2} \preccurlyeq \min(|I_h|, |J_h|) \cdot \bSigma_h^k. 
    \end{align*}It follows that $\bSigma_h^{k+1} \preccurlyeq ( 1+ \min(|I_h|, |J_h|))\cdot \bSigma_h^k $, and thus $(\bSigma_h^k)^{-1} \preccurlyeq ( 1+ \min(|I_h|, |J_h|))\cdot(\bSigma_h^{k+1})^{-1}$. Combining with \eqref{eq: proof agents regret eq 2}, we get that 
    \begin{align}\label{eq: proof agents regret eq 3}
        \sum_{k=1}^K \tilde{R}_k^M & \leq 4 \beta_u \sum_{h=1}^H \Big( \sqrt{1+ \min(|I_h|, |J_h|)}\cdot \sum_{k=1}^K \sum_{(i,j)\in X_h^k} \left\| \bPhi(C_h^k, e_h^k, i,j) \right\|_{(\bSigma_h^{k+1})^{-1}} \Big) \notag 
        \\ & \leq 4\sqrt{2} \beta_u \sum_{h=1}^H \sqrt{ \min(|I_h|, |J_h|) } \cdot \Big( \sum_{k=1}^K \sum_{(i,j)\in X_h^k} \left\| \bPhi(C_h^k, e_h^k, i,j) \right\|_{(\bSigma_h^{k+1})^{-1}}\Big),
    \end{align}where the second step is by $\min(|I_h|, |J_h|)\geq 1$. To bound the summation in the bracket, we seek to use the elliptical potential lemma. However, note tat the telescoping sum involves several utility observations (i.e. all $(i,j) \in X_h^k$) at one time, instead of a single observation. 
    To address this issue, for any $(k,h)$, let's index all the pairs $(i,j) \in X_h^k$ by an index $m=1,\cdots, |X_h^k|$. Here the order of the indexing does not matter. With a slight abuse of notation, we denote $\bPhi_h^k(m) = \bPhi(C_h^k, e_h^k, i, j)$ if $(i,j)$ has index $m$ in our indexing. 
    We also define, for $n = 1 ,\cdots, |X_h^k|$, that 
    \begin{align*}
        \bSigma_h^{k+1}(n) = \bSigma_h^k + \sum_{m=1}^n \bPhi_h^k(m) \bPhi_h^k(m)^\top.
    \end{align*}By the above definition, we have $\bSigma_h^k(n_1) \preccurlyeq \bSigma_h^k(n_2)$ for all $1\leq n_1 < n_2 \leq |X_h^k|$, and $\bSigma_h^{k+1} = \bSigma_h^{k+1}(|X_h^k|)$. It follows that 
    \begin{align*}
        \sum_{k=1}^K \sum_{(i,j)\in X_h^k} \left\| \bPhi(C_h^k, e_h^k, i,j) \right\|_{(\bSigma_h^{k+1})^{-1}} & = \sum_{k=1}^K \sum_{(i,j)\in X_h^k} \left\| \bPhi(C_h^k, e_h^k, i,j) \right\|_{(\bSigma_h^{k+1}(|X_h^k|))^{-1}} 
        \\ & = \sum_{k=1}^K \sum_{m=1}^{|X_h^k|} \left\| \bPhi_h^k(m) \right\|_{(\bSigma_h^{k+1}(|X_h^k|))^{-1}} 
        \\ & \leq \sum_{k=1}^K \sum_{m=1}^{|X_h^k|} \left\| \bPhi_h^k(m) \right\|_{(\bSigma_h^{k+1}(m))^{-1}} 
        \\ & \leq \sqrt{K \cdot \min(|I_h|, |J_h|) \cdot d^2 \log\left( \frac{K \min(|I_h|, |J_h|) + d^2}{d^2}\right)} ,
    \end{align*} where the first inequality is by $\bSigma_h^k(|X_h^k|)^{-1} \preccurlyeq \bSigma_h^k(m)^{-1}$, and the second inequality is by Lemma \ref{lem: elliptical potential} with $\lambda = 1$, $\bPhi \in \RR^{d^2}$, and $|X_h^k| \leq \min(|I_h, J_h|)$. Combining with \eqref{eq: proof agents regret eq 3}, we get that 
    \begin{align*}
        \sum_{k=1}^K \tilde{R}_k^M \leq 4\sqrt{2} \beta_u d \Big( \sum_{h=1}^H \min(|I_h|, |J_h|) \Big) \cdot \sqrt{K \log\left( \frac{K \min(|I_h |, |J_h|)+d^2}{d^2}\right)}. 
    \end{align*}The final bound follows by plugging in the expression of $\beta_u$, and using the fact that $W_h \leq \min\{|I_h|, |J_h|\}$. The $1-2\delta$ probability comes from the union bound of the event of Lemma \ref{lem: ucb for utility} and the event of \eqref{eq: proof agents regret martingale diff}.

\end{proof}

\subsection{Proof of Theorem \ref{thm: planner regret}}\label{sec: complete proof of planner regret}

To prove Theorem \ref{thm: planner regret}, we need the following two lemmas which is helpful for bounding $E_1$ and $E_2$. 

\begin{lemma}[Proof in Section \ref{sec: proof of delta negative}]\label{lem: E1 bound delta negative}
    Under the setting of Theorem \ref{thm: planner regret}, with probability at least $1-\delta$, for all $(h,k)\in[H]\times [K]$ and $(C,e) \in \cC\times\Upsilon$, it holds that 
    \begin{align*}
        -2\beta_V \cdot \| \bpsi(C,e) \|_{ {(\bLambda_h^{k})}^{-1}} \leq \delta_h^k(C,e) \leq 0.
    \end{align*}
\end{lemma}

\begin{lemma}[Proof in Section \ref{sec: proof of E2 bound}]\label{lem: E2 bound}
    For any $\delta >0$, with probability at least $1 - \delta$, it holds that 
    \begin{align*}
        \sum_{k=1}^K \sum_{h=1}^H \left( \zeta_{k,h}^1 + \zeta_{k,h}^2 \right) & \leq 3\Bigg(\sum_{h=1}^H W_h \Bigg) \sqrt{K \log2/\delta} . 
    \end{align*}
\end{lemma}

We are now ready to prove Theorem \ref{thm: planner regret}. 

\begin{proof}[Proof of Theorem \ref{thm: planner regret}]
    By Lemma \ref{lem: regret decomp planner}, we bound the three terms separately. 
    
    \paragraph{Bound on $E_3$.}    According to Algorithm \ref{alg: matching main}, the planner's policy $\pi_{k,h}$ is the greedy policy with respect to $\bar{Q}_h^k$. It follows that 
    \begin{align*}
        & \ \langle \bar{Q}_h^k(C_h,I_h, J_h \cdot), \pi_h^{\star}(\cdot | C_h,I_h, J_h)- \pi_{k,h}(\cdot | C_h,I_h, J_h) \rangle_{\Upsilon} 
        \\ 
        & \qquad = \langle \bar{Q}_h^k(C_h,I_h, J_h, \cdot), \pi_h^{\star}(\cdot | C_h,I_h, J_h)\rangle_{\Upsilon} - \max_{e \in \Upsilon} \bar{Q}_h^k(C_h,I_h, J_h, e) \leq 0.
    \end{align*}Therefore, we have $E_3 \leq 0$.  
    
    \paragraph{Bound on $E_1$.} We apply Lemma \ref{lem: E1 bound delta negative}. 
    The $E_1$ term can be bounded as
    \begin{align*}
        E_1 & \leq \sum_{k=1}^K \sum_{h=1}^H - \delta_h^k(C_h^k, e_h^k) \leq 2 \beta_V \cdot \sum_{k=1}^K \sum_{h=1}^H \left\| \bpsi(C_h^k,e_h^k) \right\|_{ {(\bLambda_h^{k})}^{-1}}.
    \end{align*} Note that $\bLambda_h^k = \lambda \bI + \sum_{t=1}^{k-1} \bpsi(C_h^t, e_h^t) \bpsi(C_h^t, e_h^t)^\top$ by definition, where $\|\bpsi\|_2\leq 1$ by Assumption \ref{assump: model and feature} and $\lambda = 1$ by our choice. Thus we have $\bLambda_h^{k+1} = \bLambda_h^k + \bpsi(C_h^t, e_h^t) \bpsi(C_h^t, e_h^t)^\top \preccurlyeq 2 \bLambda_h^k $, or equivalently, $(\bLambda_h^k)^{-1} \preccurlyeq 2 (\bLambda_h^{k+1})^{-1}$, for all $k$. Applying this to the above inequality, we get the final bound for $E_1$:
    \begin{align}\label{eq: proof E1 bound}
        E_1 & \leq 2 \beta_V \cdot \sum_{h=1}^H \sum_{k=1}^K \sqrt{2} \left\| \bpsi(C_h^k,e_h^k) \right\|_{ {(\bLambda_h^{k+1})}^{-1}} \leq 2 \sqrt{2} \beta_V H \cdot \sqrt{Kd \cdot \log\left( \frac{K+d}{d}\right)}, 
    \end{align} where the second step is by the Elliptical Potential Lemma (Lemma \ref{lem: elliptical potential}). 
    
    \paragraph{Bound on $E_2$.} By Lemma \ref{lem: E2 bound}, we have 
    \begin{align}\label{eq: proof E2 bound}
        E_2 \leq  3\Big(\sum_{h=1}^H W_h \Big) \cdot \sqrt{K \cdot \log\frac{2}{\delta}}.
    \end{align}
    
    Combining Lemma \ref{lem: regret decomp planner}, \eqref{eq: proof E1 bound}, \eqref{eq: proof E2 bound} and $E_3 \leq 0$, we get that, with probability at least $1-3\delta$, 
    \begin{align*}
        R^P(K) \leq 6 \eta d^{5/2} H \Big( \sum_{h=1}^H W_h\Big) \cdot \sqrt{K} \cdot \log \left( \frac{dKH \min\{|\cI|, |\cJ|\} }{\delta} \right),
    \end{align*}where the $1-3\delta$ probability is from the union bound on the events of Lemma \ref{lem: ucb for utility}, Lemma \ref{lem: E1 bound delta negative} and Lemma \ref{lem: E2 bound}. Since $dKH\min\{|\cI|, |\cJ|\}/\delta > 3$, replacing $\delta$ with $\delta/3$ and absorbing the constant into the big-O notation, we finish the proof.

\end{proof}

\section{Proof of Lemmas}

\subsection{Proof of Lemma \ref{lem: ucb for utility}}\label{sec: proof of ucb for utility}

\begin{proof}[Proof of Lemma \ref{lem: ucb for utility}]
    Fix arbitrary $h$. Since by Assumption \ref{assump: model and feature}, the noises in the observed utilities are independent and 1-sub-Gaussian, we can apply Theorem 2 in \citep{abbasi2011improved}. We then have that, with probability at least $1-\delta/2$, for any $k \in [K]$, 
    \begin{align*}
        & \left\| \btheta_h^k - \btheta_h \right\|_{\bSigma_h^k}  \leq \sqrt{d^2 \cdot \log \left( \frac{1+k\cdot \min\{|I_h|, |J_h|\}\cdot d^2/\lambda }{\delta/2}\right)} + \sqrt{\lambda} d.
    \end{align*} We then have that for any $C,e,i,j$, 
    \begin{align*}
        u_h^k(C,e,i,j) - u_h(C,e,i,j) 
        &= \langle \bPhi(C,e,i,j), \btheta_h^k - \btheta_h \rangle + \beta_u \|\Phi(C,e,i,j)\|_{(\bSigma_h^k)^{-1}} 
        \\ 
        &= \langle (\bSigma_h^k)^{-1/2}\bPhi(C,e,i,j), (\bSigma_h^k)^{1/2}(\btheta_h^k - \btheta_h) \rangle + \beta_u \|\Phi(C,e,i,j)\|_{(\bSigma_h^k)^{-1}} 
        \\ 
        &\geq \beta_u \|\Phi(C,e,i,j)\|_{(\bSigma_h^k)^{-1}} - \left\|\bPhi(C, e, i,j) \right\|_{(\bSigma_h^k)^{-1}} \cdot \left\| \btheta_h^k - \btheta_h \right\|_{\bSigma_h^k} 
        \\ 
        &\geq 0, 
    \end{align*}where the first inequality follows from the Cauchy-Schwarz inequality and the second in equality is my the choice of $\beta_u$. Similarly, we have
    \begin{align*}
        u_h^k(C,e,i,j) - u_h(C,e,i,j) 
        & \leq \beta_u \|\Phi(C,e,i,j)\|_{(\bSigma_h^k)^{-1}} + \left\|\bPhi(C, e, i,j) \right\|_{(\bSigma_h^k)^{-1}} \cdot \left\| \btheta_h^k - \btheta_h \right\|_{\bSigma_h^k} 
        \\ & \leq 2 \beta_u \|\Phi(C,e,i,j)\|_{(\bSigma_h^k)^{-1}} .
    \end{align*} The same argument holds for $v_h^k - v_h$ with probability at least $1- \delta/2$.
    
    Finally, we take a union bound and conclude that the event holds with probability at least $1-\delta$. 
\end{proof}

\subsection{Proof of Lemma \ref{lem: planner_optimism}}\label{sec: proof of planner optimism}
\begin{proof}[Proof of Lemma \ref{lem: planner_optimism}]
Note that $\bar r_h$ is the maximum value of the linear program \eqref{eq: linear program definition} with $(I_h, J_h, u_h, v_h)$, and $\bar r_h^k$ is the maximum value of \eqref{eq: linear program definition} with $(I_h, J_h, u_h^k, v_h^k)$. 
Since $u_h\leq  u_h^k$ and $v_h\leq v_h^k$ by Lemma \ref{lem: ucb for utility} and the weights in \eqref{eq: linear program definition} are restricted to be, it immediately holds that $\bar r_h(C,e,I_h,J_h) \leq  \bar{r}_h^k(C,e,I_h,J_h)$.

On the other hand, denote by $X_h$ the matching corresponding to $\bar{r}_h$. It follows that
\begin{align*}
    \bar r_h(C,e,I_h,J_h) &= \sum_{(i,j)\in X_h} (u_h(C,e,i,j) + v_h(C,e,i,j)) \geq \sum_{(i,j)\in X_h^k} (u_h(C,e,i,j) + v_h(C,e,i,j))\\ 
    &= \sum_{(i,j)\in X_h^k} ( u_h^k(C,e,i,j) + v_h^k(C,e,i,j))  - \sum_{(i,j)\in X_h^k} (b_{u,h}(C,e,i,j) + b_{v,h}(C,e,i,j)),
\end{align*}where the inequality is due to the sub-optimality of $X_h^k$ under $(u_h,v_h)$. The result follows by using
\begin{align*}
\bar r_h^k(C,e,I_h,J_h)  = \sum_{(i,j)\in X_h^k} ( u_h^k(C,e,i,j) + v_h^k(C,e,i,j)) ,
\end{align*}and rearranging the terms.
This completes the proof.
\end{proof}

\subsection{Proof of Lemma \ref{lem: SI bound by ucb width}}\label{sec: proof of SI bound by ucb width}
Lemma \ref{lem: SI bound by ucb width} is a restatement of Lemma 5.4 in \citet{jagadeesan2021learning}. For completeness, we present the proof here. 
Specifically, we prove a general version of Lemma \ref{lem: SI bound by ucb width}, which is Lemma \ref{lem: general version SI bound by ucb width} below.

\begin{lemma}\label{lem: general version SI bound by ucb width}
    Let $(u,v)$ and $(\hat{u}, \hat{v})$ be two pairs of utility functions on the agents set $I,J$, such that each of $u$, $v$, $\hat{u}$, $\hat{v}$ maps from $I\times J$ to $\RR$. Let $(\hat{X},\hat{\tau})$ be a stable matching on $(I,J)$ w.r.t. the utility functions $(\hat{u},\hat{v})$. Suppose $u\leq \hat{u}$, $v\leq\hat{v}$. Then the Subset Instability of $(\hat{X}, \hat{\tau})$ w.r.t. the utility $(u,v)$ satisfies
    \begin{align*}
        \SI(\hat{X},\hat{\tau}; I,J, u , v) \leq \sum_{i \in I} \left|\hat{u}(i, \hat{X}(i)) - u(i, \hat{X}(i))\right| + \sum_{j \in J} \left|\hat{v}(\hat{X}(j),j) - v(\hat{X}(j),j) \right|.
    \end{align*}
\end{lemma}

\begin{proof}[Proof of Lemma \ref{lem: general version SI bound by ucb width}]
    Define the function 
    \begin{align*}
        & f(I', J', X,\tau; u,v) 
        \\ &\qquad =  \max_{X'} \bigg(\sum_{i \in I'}  u(i, X'(i)) + \sum_{j \in J'}  v(X'(j), j) \bigg)  - \sum_{i\in I'} \left( u(i,X(i)) + \tau(i) \right) - \sum_{j\in J'} \left( v(X(j),j) + \tau(j) \right),
    \end{align*}where $I'\times J' \subset I \times J$.
    Then by Definition \ref{def: subset instatbility}, $\SI(X,\tau; I,J, u ,v) = \max_{I' \times J' \subset I \times J } f(I', J', X,\tau; u,v)$, thus
    \begin{align*}
        \SI(\hat{X}, \hat{\tau}; I,J,u,v ) -  \SI(\hat{X}, \hat{\tau}; I,J,\hat{u},\hat{v} ) \leq \max_{I'\times J' \subset I \times J} \left[ f(I', J', \hat{X}, \hat{\tau}; u,v) - f(I', J', \hat{X}, \hat{\tau}; \hat{u}, \hat{v}) \right].
    \end{align*}To bound $f(I', J', \hat{X}, \hat{\tau}; u,v) - f(I', J', \hat{X}, \hat{\tau}; \hat{u}, \hat{v})$, we decompose
    \begin{align*}
        & f(I', J', \hat{X}, \hat{\tau}; u,v) - f(I', J', \hat{X}, \hat{\tau}; \hat{u}, \hat{v}) 
        \\ & \qquad = \underbrace{ \max_{X'} \bigg(\sum_{i \in I'}  u(i, X'(i)) + \sum_{j \in J'}  v(X'(j), j) \bigg) - \max_{X'} \bigg(  \sum_{i \in I'}  \hat{u}(i, X'(i)) + \sum_{j \in J'}  \hat{v}(X'(j), j) \bigg)}_{\textnormal{I}}\\ 
        & \quad \qquad + \underbrace{ \sum_{i\in I'} \left( \hat{u}(i,\hat{X}(i)) + \hat{\tau}(i) \right) + \sum_{j\in J'} \left( \hat{v}(\hat{X}(j),j) + \hat{\tau}(i) \right)- \sum_{i\in I'} \left( u(i,\hat{X}(i)) + \hat{\tau}(i) \right)  - \sum_{j\in J'} \left( v(\hat{X}(j),j) + \hat{\tau}(j) \right)}_{\textnormal{II}}
    \end{align*}
    Term I is nonpositive. To see this, note that by assumption $u \leq \hat{u}$ and $v \leq \hat{v}$. Thus the max-weight matching on $I'\times J'$ w.r.t. the utility functions $(u,v)$ cannot exceed the max-weight matching on $I'\times J'$ w.r.t. $(\hat{u}, \hat{v})$. 
    
    To bound term II, note that all the transfers $\hat{\tau}(i)$ and $\hat{\tau}(j)$ in the expression cancel out, and it follows that 
    \begin{align*}
        \text{II} &= \sum_{i\in I'} \left( \hat{u}(i,\hat{X}(i)) - u(i,\hat{X}(i)) \right) + \sum_{j\in J'} \left( \hat{v}(\hat{X}(j),j) - v(\hat{X}(j),j) \right)\\   
        &\leq \sum_{i\in I} \left( \hat{u}(i,\hat{X}(i)) - u(i,\hat{X}(i)) \right) + \sum_{j\in J} \left( \hat{v}(\hat{X}(j),j) - v(\hat{X}(j),j) \right) 
    \end{align*}
    where the inequality follows from the assumption that $u\leq \hat u$ and $v\leq \hat v$ and $I'\times J'\subseteq I\times J$.
    This finishes the proof. 
\end{proof}

\begin{proof}[Proof of Lemma \ref{lem: SI bound by ucb width}] 
    For any fixed $k\in[K]$ and $h\in[H]$, replace $I,J$ in Lemma \ref{lem: general version SI bound by ucb width} with $I_h, J_h$, and replace $u$, $v$, $\hat{u}$, $\hat{v}$ in with $u_h(C_h^k, e_h^k, \cdot, \cdot)$, $v_h(C_h^k, e_h^k, \cdot, \cdot)$, $u_h^k(C_h^k, e_h^k, \cdot, \cdot)$, $v_h^k(C_h^k, e_h^k, \cdot, \cdot)$. Since $X_h^k$ is the stable matching w.r.t. $u_h^k(C_h^k, e_h^k, \cdot, \cdot)$, $v_h^k(C_h^k, e_h^k, \cdot, \cdot)$, we can replace $\hat{X}$ with $X_h^k$. It then follows from Lemma \ref{lem: general version SI bound by ucb width} that
    \begin{align*}
        \SI_h^k & \leq \sum_{i \in I_h} \left|u_h^k(C_h^k, e_h^k, i, X_h^k(i)) - u_h(C_h^k, e_h^k, i, X_h^k(i))\right| + \sum_{j \in J_h} \left|v_h^k(C_h^k, e_h^k, X_h^k(j),j) - v_h(C_h^k, e_h^k,, X_h^k(j),j) \right|
        \\ & = \sum_{(i,j)\in X_h^k} \left|u_h^k(C_h^k, e_h^k, i, j) - u_h(C_h^k, e_h^k, i, j)\right| + \left|v_h^k(C_h^k, e_h^k, i, j) - v_h(C_h^k, e_h^k, i, j)\right|
        \\ &  = \sum_{(i,j)\in X_h^k} \left(  b_{u,h}(C_h^k, e_h^k,i,j) + b_{v,h}(C_h^k,e_h^k,i,j) \right),
    \end{align*}where the second step holds because the true utility and the estimated utility are zero for unmatched agents under $X_h^k$, and the last step is by the definition of the bonus function $b_{u,h}$ and $b_{v,h}$. 
\end{proof}

\subsection{Proof of Planner's Regret Decomposition}\label{sec: proof of regret decomp planner}

We first restate the lemma in its complete form. 
\begin{replemma}{lem: regret decomp planner}
    The planner's regret defined by \eqref{eq: regret planner} can be decomposed as
    \begin{align*}
        R^P(K) & = \underbrace{\sum_{k=1}^K \sum_{h=1}^H \left[ \EE_{\pi^{\star}}\left[\delta_h^k(C_h, I_h, J_h, e_h) \mid C_1 = C_1^k \right] - \delta_h^k(C_h^k, I_h, J_h, e_h^k) \right] }_{E_1} + \underbrace{\sum_{k=1}^K \sum_{h=1}^H \left( \zeta_{k,h}^1 + \zeta_{k,h}^2 \right)}_{E_2} \notag
        \\ 
        & \qquad\qquad + \underbrace{\sum_{k=1}^K \sum_{h=1}^H \EE_{\pi^{\star}} \left[ \langle \bar{Q}_h^k(C_h, I_h, J_h,\cdot), \pi_h^{\star}(\cdot | C_h, I_h, J_h)- \pi_{k,h}(\cdot | C_h, I_h, J_h) \rangle_{\Upsilon} \middle| C_1 = C_1^k \right]}_{E_3},
    \end{align*}where the expectation is over the trajectory $\{C_h,e_h\}_{h\in[H]}$ induced by executing the policy $\pi^{\star}$ (on the choice of $e\in\Upsilon$ only), and conditioning on $\{I_h, J_h\}_{h\in[H]}$ being fixed.
\end{replemma}

\begin{proof}[Proof of Lemma \ref{lem: regret decomp planner}]
    Recall the definition of the planner's regret from \eqref{eq: regret planner}. We write
    \begin{align*}
        \bar{V}_1^{\star}(s_1^k) - \bar{V}_1^{\pi_k} (s_1^k) & = \underbrace{\bar{V}_1^{\star}(s_1^k) - \bar{V}_1^k(s_1^k)}_{\textbf{I}} + \underbrace{\bar{V}_1^k(s_1^k)- \bar{V}_1^{\pi_k} (s_1^k)}_{\textbf{II}},
    \end{align*} where $s_h^k = (C_h^k, I_h, J_h)$ by our notation.
    
    \paragraph{Term I.} 
    We define two operators $\JJ_h^{\star}$ and $\JJ_{k,h}$ as
    \begin{align*}
        \JJ_h^{\star} Q (s) = \langle Q(s, \cdot) , \pi^{\star}_h(\cdot | s) \rangle_{\Upsilon} , \quad \JJ_{k,h}^{\star} Q (s) = \langle Q(s, \cdot) , \pi_{k,h}(\cdot | s) \rangle_{\Upsilon},
    \end{align*}for all $(k,h) \in [H]\times[K]$, $s\in \cC \times 2^\cI\times 2^\cJ$, and function $Q: \cC\times \cI \times \cJ \times \Upsilon \to \RR$. Then by definition, we have $\bar{V}_h^k = \JJ_{k,h} \bar{Q}_h^k$, and $\bar{V}_h^{\star} = \JJ_h^{\star} \bar{Q}_h^{\star}$. 
    It follows that
    \begin{align}\label{eq: decomp term 1 eq 1}
        \bar{V}_h^{\star} - \bar{V}_h^k & = \JJ_h^{\star} \bar{Q}_h^{\star} - \JJ_{k,h} \bar{Q}_h^k = \left( \JJ_h^{\star} \bar{Q}_h^{\star} - \JJ_h^{\star} \bar{Q}_h^k\right) + \left( \JJ_h^{\star} \bar{Q}_h^k - \JJ_{k,h} \bar{Q}_h^k \right) = \left( \JJ_h^{\star} \bar{Q}_h^{\star} - \JJ_h^{\star} \bar{Q}_h^k\right) + \xi_h^k,
    \end{align}
    where $\xi_h^k \coloneqq \JJ_h^{\star} \bar{Q}_h^k - \JJ_{k,h} \bar{Q}_h^k$. 
    Also, by the definition of $\delta_h^k$ in \eqref{eq: define terms in regret}, we have
    \begin{align*}
        \bar{Q}_h^{\star} - \bar{Q}_h^k & = \bar{r}_h + \PP_h \bar{V}_{h+1}^{\star} - \left( \bar{r}_h + \PP_h \bar{V}_{h+1}^k \right) + \delta_h^k = \PP_h \left( \bar{V}_{h+1}^{\star} - \bar{V}_{h+1}^k \right) + \delta_h^k.
    \end{align*}Combining with \eqref{eq: decomp term 1 eq 1}, we get 
    \begin{align*}
        \bar{V}_h^{\star} - \bar{V}_h^k & = \JJ_h^{\star} \PP_h \left( \bar{V}_{h+1}^{\star} - \bar{V}_{h+1}^k \right) + \JJ_h^{\star} \delta_h^k + \xi_h^k.
    \end{align*}Applying the above equation recursively, we have that for any $C_1, I_1, J_1$, 
    \begin{align}\label{eq: decomp term 1 eq 2}
        & \bar{V}_1^{\star} (C_1,I_1,J_1) - \bar{V}_1^k(C_1, I_1, J_1) \notag 
        \\ 
        & \qquad = \prod_{h=1}^H \left(\JJ_h^{\star} \PP_h \right) \left(\bar{V}_{H+1}^{\star} - \bar{V}_{H+1}^k \right) (C_1,I_1,J_1) \notag
        \\ & \qquad \qquad + \sum_{h=1}^H \left( \prod_{l=1}^{h-1}\JJ_l^{\star} \PP_l \right) \JJ_h^{\star} \delta_h^k (C_1,I_1,J_1) + \sum_{h=1}^H \left( \prod_{l=1}^{h-1}\JJ_l^{\star} \PP_l \right) \xi_h^k  (C_1,I_1,J_1) \notag
        \\ 
        & \qquad = \EE_{\pi^{\star}} \left[ \sum_{h=1}^H \delta_h^k (C_h,I_h, J_h, e_h) \middle| C_1, I_1, J_1 \right] \notag
        \\ & \qquad \qquad + \EE_{\pi^{\star}} \left[ \sum_{h=1}^H \langle \bar{Q}_h^k(C_h,I_h, J_h, \cdot), [\pi_h^{\star} - \pi_{k,h}](\cdot \mid C_h, I_h, J_h) \rangle_{\Upsilon} \middle| C_1, I_1, J_1 \right] , 
    \end{align} where the second step holds because $\bar{V}_{H+1}^{\star} = \bar{V}_{H+1}^k=0$. By definition of the operators, it is clear that here the expectation $\EE_{\pi^{\star}}$ is over the trajectory $\{(C_h, e_h)\}_{h \in [H]}$ induced by the planner executing the policy $\pi^{k}$ to choose actions in $\Upsilon$. 
    
    \paragraph{Term II.}
    First note that by \eqref{eq: define terms in regret}, the function $\delta_h^k: \cC \times 2^\cI \times 2^\cJ \times \Upsilon \to \RR$ can be written as 
    \begin{align}\label{eq: term 2 delta alt form}
        \delta_h^k & = \bar{r}_h + \PP_h \bar{V}_{h+1}^k - \bar{Q}_h^k = \bar{r}_h + \PP_h \bar{V}_{h+1}^k - \bar{Q}_h^{\pi_k} + \bar{Q}_h^{\pi_k} - \bar{Q}_h^k = \PP_h\left( \bar{V}_{h+1}^k - \bar{V}_h^{\pi_k} \right) + \left( \bar{Q}_h^{\pi_k} - \bar{Q}_h^k \right),
    \end{align} where the last step is by $\bar{Q}_h^{\pi_k} = \bar{r}_h + \PP_h V_{h+1}^{\pi_k}$. Then for any $h$, we can write
    \begin{align*}
        \left[\bar{V}_h^k - \bar{V}_h^{\pi_k} \right] (C_h^k, I_h, J_h) &= \left[ \bar{V}_h^k - \bar{V}_h^{\pi_k}  + \delta_h^k - \delta_h^k\right] (C_h^k, I_h, J_h)\\ 
        &= \left[ \bar{V}_h^k - \bar{V}_h^{\pi_k} \right] (C_h^k, I_h, J_h) + \left[\bar{Q}_h^{\pi_k} - \bar{Q}_h^k \right](C_h^k, I_h, J_h, e_h^k) \\ 
        & \qquad\qquad + \PP_h \left[\bar{V}_{h+1}^k - \bar{V}_{h+1}^{\pi_k} \right] (C_h^k, I_h, J_h, e_h^k) - \delta_h^k (C_h^k, I_h, J_h, e_h^k)
        \\ 
        &= \left[ \bar{V}_h^k - \bar{V}_h^{\pi_k} \right] (C_h^k, I_h, J_h) - \left[ \bar{Q}_h^k - \bar{Q}_h^{\pi_k}  \right](C_h^k, I_h, J_h, e_h^k) 
        \\ 
        & \qquad\qquad + \PP_h \left[\bar{V}_{h+1}^k - \bar{V}_{h+1}^{\pi_k} \right] (C_h^k, I_h, J_h, e_h^k) - [\bar{V}_{h+1}^k - \bar{V}_{h+1}^{\pi_k} ] (C_{h+1}^k, I_{h+1}, J_{h+1})\\
        &\qquad\qquad + [\bar{V}_{h+1}^k - \bar{V}_{h+1}^{\pi_k} ] (C_{h+1}^k, I_{h+1}, J_{h+1}) - \delta_h^k (C_h^k, I_h, J_h, e_h^k) \\ 
        &= [ \bar{V}_{h+1}^k - \bar{V}_{h+1}^{\pi_k} ] (C_{h+1}^k, I_{h+1}, J_{h+1}) - \delta_h^k (C_h^k, I_h, J_h, e_h^k) 
        \\ 
        &\qquad\qquad + \underbrace{\left[ \PP_h[\bar{V}_h^k - \bar{V}_h^{\pi_k}] (C_h^k, I_h,J_h, e_h^k) - [ \bar{V}_{h+1}^k - \bar{V}_{h+1}^{\pi_k}](C_{h+1}^k, I_{h+1}, J_{h+1}) \right] }_{\zeta_{k,h}^2}
        \\ 
        &\qquad\qquad + \underbrace{ [ \bar{V}_h^k - \bar{V}_h^{\pi_k} ] (C_h^k, I_h, J_h) - [\bar{Q}_h^k - \bar{Q}_h^{\pi_k}] (C_h^k, I_h,J_h, e_h^k) }_{\zeta_{k,h}^1} , 
    \end{align*} where the second step is by \eqref{eq: term 2 delta alt form}. Applying the above equation recursively, we get 
    \begin{align}\label{eq: decomp term 2 eq 1}
        [\bar{V}_1^k - \bar{V}_1^{\pi_k}] (C_1^k, I_1, J_1) & = \sum_{h=1}^H \left( \zeta_{k,h}^1 + \zeta_{k,h}^2 \right) - \sum_{h=1}^H \delta_h^k (C_h^k, I_h, J_h, e_h^k),
    \end{align}where we use $\bar{V}_{H+1}^{\star} = \bar{V}_{H+1}^k=0$ again. 
    
    Combining \eqref{eq: decomp term 1 eq 2} and \eqref{eq: decomp term 2 eq 1}, we get 
    \begin{align*}
        & \sum_{k=1}^K \left[ \bar{V}_1^{\star} (C_1^k, I_1, J_1) - \bar{V}_1^{\pi_k} (C_1^k, I_1, J_1) \right] 
        \\ &\qquad = \sum_{k=1}^K \sum_{h=1}^H \EE_{\pi^{\star}} \left[  \delta_h^k (C_h,I_h, J_h, e_h) \middle| C_1, I_1, J_1 \right] - \sum_{k=1}^K \sum_{h=1}^H \delta_h^k (C_h^k, I_h, J_h, e_h^k) 
        \\ 
        &\qquad \qquad+ \sum_{k=1}^K \sum_{h=1}^H \left( \zeta_{k,h}^1 + \zeta_{k,h}^2 \right)
        \\ 
        &\qquad \qquad + \EE_{\pi^{\star}} \left[ \sum_{h=1}^H \langle \bar{Q}_h^k(C_h,I_h, J_h, \cdot), [\pi_h^{\star} - \pi_{k,h}](\cdot \mid C_h, I_h, J_h) \rangle_{\Upsilon} \middle| C_1, I_1, J_1 \right] ,
    \end{align*}which finishes the proof. 
    
\end{proof}

\subsection{Proof of Lemma \ref{lem: E1 bound delta negative}}\label{sec: proof of delta negative}

\begin{proof}
    
By the definition of $\bar{Q}_h^k$ in Algorithm \ref{alg: matching main}, the function $\delta_h^k$ satisfies 
    \begin{align*}
        \delta_h^k(C, I,J, e) & = \bar{r}_h(C,I,J,e) + \PP_h \bar{V}_{h+1}^k(C, I,J,e) - \bar{Q}_h^k (C, I,J, e)
        \\ & = \bar{r}_h(C,I,J,e) + \PP_h \bar{V}_{h+1}^k(C,I,J,e) - \bar{r}_h^k(C,I,J,e) - \hat{\PP}_h \bar{V}_{h+1}^k (C, I,J, e).
    \end{align*} 

In the sequel, we will show that $\bar{r}_h^k$ and $\hat{\PP}_h \bar{V}_{h+1}^k$ upper bound $\bar{r}_h$ and $\PP_h \bar{V}_{h+1}^k$ respectively.

We first consider the term $\bar{r}_h(C,e, I_h, J_h) - \bar{r}_h^k(C,e, I_h, J_h)$. It immediately follows from Lemma \ref{lem: planner_optimism} that, under the event of Lemma \ref{lem: ucb for utility}, 
\begin{align}\label{eq: E1 bound r}
    -4 \sum_{(i,j)\in X_h^k}  \beta_u \left\|\bPhi(C, e, i,j) \right\|_{(\bSigma_h^k)^{-1}} &\leq - \sum_{(i,j)\in X_h^k} (b_{u,h}(C,e,i,j) + b_{v,h}(C,e,i,j)) \notag
    \\ & \leq \bar{r}_h(C,e, I_h, J_h) - \bar{r}_h^k(C,e, I_h, J_h)\notag\\
    &\leq 0. 
\end{align}

We now consider $\hat{\PP}_h \bar{V}_{h+1}^k$. Since we are conditioning on $\{I_h, J_h\}$ which is independent of anything else, we can essentially treat it as a deterministic sequence. By \eqref{eq: context transition model}, for any $\bar{V}_{h+1}^k$, there exists $\bar{\wb}_h^k \in \RR^d$ such that for any $C,e$, $\PP_h \bar{V}_{h+1}^k(C,e, I_h, J_h) = \bpsi(C,e)^\top \bar{\wb}_h^k$. 
This is because 
\begin{align*}
    \PP_h \bar{V}_{h+1}^k (C,e, I_h, J_h) & = \int \bar{V}_{h+1}^k (C', I_{h+1}, J_{h+1}) \diff \PP(C'| C,e)
    \\ & = \int \bar{V}_{h+1}^k (C', I_{h+1}, J_{h+1}) \langle \bpsi(C,e), \diff \bmu_h(C')\rangle 
    \\ & = \langle \bpsi(C,e) , \int \bar{V}_{h+1}^k (C', I_{h+1}, J_{h+1})  \diff \bmu_h(C') \rangle. 
\end{align*}

Then we can write
\begin{align}\label{eq: E1 bound PV main}
    & \PP_h \bar{V}_{h+1}^k (C,e, I_h, J_h) - \hat{\PP}_h \bar{V}_{h+1}^k (C,e, I_h, J_h) \notag
    \\ 
    & \qquad = \bpsi(C,e)^\top \bar{\wb}_h^k - \bpsi(C,e)^\top \left( \bLambda_h^k \right)^{-1} \sum_{t=1}^{k-1} \bpsi(C_h^t, e_h^t) \bar{V}_{h+1}^k (C_{h+1}^t, I_{h+1}, J_{h+1}) - \beta_V \cdot \| \bpsi(C,e) \|_{ {(\bLambda_h^{k})}^{-1}} \notag
    \\ 
    & \qquad = \bpsi(C,e)^\top \left( \bLambda_h^k \right)^{-1} \left[ \bLambda_h^k \bar{\wb}_h^k - \sum_{t=1}^{k-1} \bpsi(C_h^t, e_h^t) \bar{V}_{h+1}^k (C_{h+1}^t, I_{h+1}, J_{h+1}) \right] - \beta_V \cdot \| \bpsi(C,e) \|_{ {(\bLambda_h^{k})}^{-1}}\notag
    \\ 
    & \qquad = \bpsi(C,e)^\top \left(\bLambda_h^k \right)^{-1} \left[ \sum_{t=1}^{k-1} \bpsi(C_h^t, e_h^t) \left( \PP_h \bar{V}_{h+1}^k(C_h^t, e_h^t, I_h, J_h) - \bar{V}_{h+1}^k(C_{h+1}^t, I_{h+1}, J_{h+1}) \right) \right] \notag 
    \\ 
    & \qquad\qquad + \lambda \ \bpsi(C,e)^\top \left(\bLambda_h^k \right)^{-1} \bar{\wb}_h^k - \beta_V \cdot \| \bpsi(C,e) \|_{ {(\bLambda_h^{k})}^{-1}} \,,
\end{align}where the first step uses the construction of $\wb_h^k$ in Algorithm \ref{alg: matching main} and the last step uses the construction of $\bLambda_h^k$. It follows from the Cauchy-Schwarz inequality that the first part on the R.H.S. of \eqref{eq: E1 bound PV main} satisfies
\begin{align}\label{eq: eq: E1 bound PV CS sum}
    &\left| \bpsi(C,e)^\top \left(\bLambda_h^k \right)^{-1} \left[ \sum_{t=1}^{k-1} \bpsi(C_h^t, e_h^t) \left( \PP_h \bar{V}_{h+1}^k(C_h^t, e_h^t, I_h, J_h) - \bar{V}_{h+1}^k(C_{h+1}^t, I_{h+1}, J_{h+1}) \right) \right] + \lambda \ \bpsi(C,e)^\top \left(\bLambda_h^k \right)^{-1} \bar{\wb}_h^k \right| \notag
    \\ 
    & \qquad \leq \left\| \bpsi(C,e)\right\|_{(\bLambda_h^k)^{-1}} \cdot \left\| \sum_{t=1}^{k-1} \bpsi(C_h^t, e_h^t, I_h, J_h) \left( \PP_h \bar{V}_{h+1}^k(C_h^t, e_h^t, I_h, J_h) - \bar{V}_{h+1}^k(C_{h+1}^t, I_{h+1}, J_{h+1}) \right)  \right\|_{(\bLambda_h^k)^{-1}} \notag 
    \\  
    & \qquad\qquad + \lambda \ \left\| \bpsi(C,e)\right\|_{(\bLambda_h^k)^{-1}} \cdot \left\| \bar{\wb}_h^k\right\|_{(\bLambda_h^k)^{-1}} .
\end{align}In the following, to bound \eqref{eq: E1 bound PV main}, we first bound the self-normalized stochastic process using tools from self-normalized martingale. The issue is that, according to Algorithm \ref{alg: matching main}, the function $\bar{V}_{h+1}^k$ depends on the first $(k-1)$ episodes and thus depends on the trajectory $\{(C_h^t, e_h^t, C_{h+1}^t)\}_{t \in [k-1]}$. We thus adopt a common approach to solve this issue by considering the function class containing each value function estimator $\bar{V}_{h+1}^k$. We discuss the detail of the construction of the function class and its covering in Section \ref{sec: detail of covering}. 

The covering trick allows us to get the following lemma.

\begin{lemma}\label{lem: self normalized term}
    Under the setting of Theorem \ref{thm: planner regret}, with probability at least $1-\delta$, for any $(h,k)\in[H]\times [K]$, 
    \begin{align*}
        & \left\| \sum_{t=1}^{k-1} \bpsi(C_h^t, e_h^t) \left( \PP_h \bar{V}_{h+1}^k(C_h^t, e_h^t, I_h, J_h) - \bar{V}_{h+1}^k(C_{h+1}^t, I_{h+1}, J_{h+1}) \right)  \right\|_{(\bLambda_h^k)^{-1}} 
        \\ 
        & \qquad \leq 16 d^2 \cdot \left( \sum_{h=1}^{H} W_h\right) \cdot \sqrt{ \log\left( \frac{dKH \min\{|\cI|, |\cJ|\}\cdot (\beta_V + \beta_u)}{\delta}  \right) }. 
    \end{align*}
\end{lemma}

\begin{proof}
    See Appendix \ref{apdx: proof of lemma self normalized term} for the proof. 
\end{proof}

For the term $\left\| \bar{\wb}_h^k\right\|_{(\bLambda_h^k)^{-1}}$, by Assumption \ref{assump: model and feature} and $|\bar{V}_h^k| \leq H \min\{|\cI|,|\cJ|\}$, we have 
\begin{align*}
    \left\| \bar{\wb}_h^k\right\|_2 = \left\| \int_\cC \bar{V}_h^k(C, I_h, J_h) \ \diff \mu_h(C) \right\|_2 \leq \sqrt{d} \cdot \left( \sum_{h=1}^{H} W_h \right).
\end{align*}

Combine the above inequality with Lemma \ref{lem: self normalized term} and \eqref{eq: eq: E1 bound PV CS sum}, and we get that 
\begin{align*}
    &\left| \bpsi(C,e)^\top \left(\bLambda_h^k \right)^{-1} \left[ \sum_{t=1}^{k-1} \bpsi(C_h^t, e_h^t) \left( \PP_h \bar{V}_{h+1}^k(C_h^t, e_h^t, I_h, J_h) - \bar{V}_{h+1}^k(C_{h+1}^t, I_{h+1}, J_{h+1}) \right) \right] + \lambda \ \bpsi(C,e)^\top \left(\bLambda_h^k \right)^{-1} \bar{\wb}_h^k \right| \notag
    \\ 
    & \qquad \leq 17 d^2 \cdot \left(\sum_{h=1}^{H} W_h\right) \sqrt{\chi} \cdot \left\|\bpsi(C,e) \right\|_{(\bLambda_h^k)^1} , 
\end{align*}where
\begin{align*}
    \chi \coloneqq \log\left( \frac{dKH \min\{|\cI|, |\cJ|\}\cdot (\beta_V + \beta_u)}{\delta}  \right). 
\end{align*}
It remains to show that there exists choice of $\beta_V$ (or equivalently, the constant $\eta$ in the description of Theorem \ref{thm: planner regret}) such that 
\begin{align*}
     17 d^2 \cdot \left( \sum_{h=1}^{H} W_h \right) \sqrt{\chi} \leq \beta_V. 
\end{align*}Specifically, we show that we can pick some constant $\eta$ and set
\begin{align*}
    \beta_V = \eta d^2 \left( \sum_{h=1}^H W_h\right) \cdot \sqrt{\log \frac{dKH \min\{|\cI|, |\cJ|\} }{\delta}} . 
\end{align*} Indeed, plug in the expression of $\beta_V$ and $\beta_u$ and we get
\begin{align*}
    {\chi} & \leq { \log\left( \frac{3d^3 KH^2 \min\{|\cI|, |\cJ|\}^2 \cdot \eta \log\left( 3dKH \min\{|\cI|, |\cJ|\} / \delta \right) }{\delta}  \right) } 
    \\ & \leq  {3\log \left( \frac{ 3\eta dKH \min\{|\cI|, |\cJ|\} \cdot \iota }{\delta} \right)} 
    \\ & = 3 \iota + 3 \log(\eta)+ 3 \log\left( \iota \right)  ,
\end{align*}where $\iota = \log(3dKH\min\{|\cI|,|\cJ|\}/\delta)$. Since $\iota > \log 3$, it suffices to pick $\eta$ such that 
\begin{align*}
    13\sqrt{ 3 \log 3 + 3 \log(\eta)+ 3 \log\left( \log 3 \right) } & \leq \eta \cdot \log 3,
\end{align*} which finishes the proof. 

\end{proof}

\subsection{Proof of Lemma \ref{lem: E2 bound}}\label{sec: proof of E2 bound}

\begin{proof}[Proof of Lemma \ref{lem: E2 bound}]
    The lemma can be proven by standard martingale concentration, similar to the analysis in \citep{cai2020provably, yang2020function}. Specifically, we define the $\sigma$-fields as 
    \begin{align*}
        \cF_{k,h,0} & = \sigma\left( \left\{(C_l^t, e_l^t)_{(l,t)\in[k-1]\times [H]} \right\} \cup \left\{(C_l^k, e_l^k) \right\}_{l \in [H]} \right) , 
        \\ \cF_{k,h,1} & = \sigma\left( \left\{(C_l^t, e_l^t)_{(l,t)\in[k-1]\times [H]} \right\} \cup \left\{(C_l^k, e_l^k) \right\}_{l \in [H]} \cup \{C_{h+1}^k\}\right).
    \end{align*} By definition, it is clear that these $\sigma$-fields form a filtration under the dictionary order on the index tuple $(k,h,o)$ where $o\in\{0,1\}$. 
    
    Note that for any $(k,h) \in [K]\times[H]$, since $\bar{V}_h^k$, $\bar{Q}_h^k$ and the policy $\pi_k$ are all functions of the first $(k-1)$ episodes, they are all $\cF_{k,1,1}$-measurable. As a result, $\zeta_{k,h}^1$ is $\cF_{k,h,1}$-measurable and  $\zeta_{k,h}^2$ is $\cF_{k,h,2}$-measurable, for all $(k,h)$. 
    
    According to Algorithm \ref{alg: matching main}, the planner's action $e_h^k \sim \pi_k(\cdot | C_h^k)$. This indicates that condition on $C_h^k$, $\bar{V}_h^{\pi_k}(C_h^k, I_h, J_h) - \bar{Q}_h^{\pi_k}(C_h^k, e_h^k, I_h, J_h) = 0$. Also, since $e_h^k \leftarrow \argmax_{e \in \Upsilon} \bar{Q}_h^k (s_h^k , e, I_h, J_h)$, and $\bar{V}_h^k (C,I,J) = \max_{e \in \Upsilon} \bar{Q}_h^k(C,e,I,J)$ for all $(C, I, J)$ by the algorithm, we have $\bar{V}_h^k(C_h^k, I_h, J_h) - \bar{Q}_h^k(C_h^k, e_h^k, I_h, J_h) = 0$. Altogether we have $\zeta_{k,h}^1 = 0$ for all $(k,h) \in [K]\times[H]$. 
    
    For $\zeta_{k,h}^2$, first note that there is no dependence issue between the value functions and the trajectory since $\bar{V}_{h+1}^k$ and $\bar{V}_{h+1}^{\pi_k}$ are functions of the first $(k-1)$ episodes. Then since $C_{h+1}^k \sim \PP_h(\cdot| C_h^k, e_h^k)$, we have 
    \begin{align}\label{eq: zeta2 is martingale diff}
        \EE \left[ \zeta_{k,h}^2 \mid \cF_{k,h,1} \right] = 0.
    \end{align} Thus we conclude that $\{(\zeta_{k,h}^1, \zeta_{k,h}^2)\}_{(k,h)\in[K]\times[H]}$ is a martingale difference sequence. Since $\bar{V}_h^k$, $\bar{Q}_h^k$, $\bar{V}_h^{\pi_k}$, $\bar{Q}_h^{\pi_k}$ are all bounded by $W_h$ for all $h,k$, we apply the Azuma-Hoeffding inequality (Lemma \ref{lem: azuma hoeffding vanilla}) and get that, 
    \begin{align*}
        \PP\left( \left| E_2\right| \geq \epsilon \right)& \leq 2 \exp\left( \frac{- \epsilon^2}{8 K\sum_{h = 1}^H W_h^2} \right) . 
    \end{align*}Equivalently, with probability at least $1-\delta$, we have 
    \begin{align*}
        |E_2| & \leq \sqrt{8 K \cdot \log\frac{2}{\delta}} \cdot \sqrt{\sum_{h=1}^H W_h^2} \leq \sqrt{8 K \cdot \log\frac{2}{\delta}} \cdot \Big(\sum_{h=1}^H W_h \Big) ,
    \end{align*}
    which finishes the proof.

\end{proof}

\subsection{Proof of Lemma \ref{lem: self normalized term}}\label{apdx: proof of lemma self normalized term}
\begin{proof}[Proof of Lemma \ref{lem: self normalized term}]
    By the analysis in Section \ref{sec: detail of covering}, there exists a function class $\cV$ containing all $\bar{V}_h^k$, and the $\epsilon$-covering number of $\cV$ is given by Lemma \ref{lem: covering number of V}. Also note that by the truncation, we have $|\bar{V}_h^k| \leq \sum_{l=h}^{H} W_l $. Then we apply Lemma \ref{lem: self normalized for V function class raw} with $R = \sum_{l=h}^{H} W_l$ and combine with Lemma \ref{lem: covering number of V}, and get that, fix any $0<\epsilon < 1$, with probability at least $1-\delta/H$, for all $k\in[K]$, 
    \begin{align*}
        & \left\| \sum_{t=1}^{k-1} \bpsi(C_h^t, e_h^t) \left( \PP_h \bar{V}_{h+1}^k(C_h^t, e_h^t, I_h, J_h) - \bar{V}_{h+1}^k(C_{h+1}^t, I_{h+1}, J_{h+1}) \right)  \right\|^2_{(\bLambda_h^k)^{-1}} 
        \\ 
        & \qquad \leq 4 \Big(\sum_{l=h}^{H} W_l \Big)^2 \bigg[ \frac{d}{2}\log\left(\frac{ k+\lambda}{\lambda} \right) +6 d^2 \log \left( 1 + \frac{dkH \min\{|\cI|, |\cJ|\} \cdot \beta_V }{\min\{\lambda, 1\} \cdot \epsilon  } \right) 
        \\ 
        & \qquad\qquad + 4 d^4 \log \left( 1 + \frac{d \min\{|\cI|, |\cJ| \} \cdot \beta_u  }{\min\{\lambda, 1\} \cdot \epsilon } \right) + \log \frac{1}{\delta} \bigg]  + \frac{8 k^2 \epsilon^2}{\lambda} . 
    \end{align*}
    Let $\lambda=1$, pick $\epsilon = d^2 \left( \sum_{l=h}^{H} W_l \right) /K$ and then take a union bound over $h\in[H]$, we get 
    \begin{align*}
        & \left\| \sum_{t=1}^{k-1} \bpsi(C_h^t, e_h^t) \left( \PP_h \bar{V}_{h+1}^k(C_h^t, e_h^t, I_h, J_h) - \bar{V}_{h+1}^k(C_{h+1}^t, I_{h+1}, J_{h+1}) \right)  \right\|_{(\bLambda_h^k)^{-1}} 
        \\ 
        & \qquad \leq 16 d^2 \cdot\left(\sum_{l=h}^{H} W_l\right) \cdot \sqrt{ \log\left( \frac{dKH \min\{|\cI|, |\cJ|\}\cdot (\beta_V + \beta_u)}{\delta}  \right) }. 
    \end{align*} 

\end{proof}

\section{Covering Number of Function Classes}\label{sec: detail of covering}
In this section, we will construct a function class $\cV$ that provably contains $\bar{V}_h^k$ for all $(k,h)\in[K]\times[H]$. 
And we will compute the covering number of $\cV$. The result is summarized by Lemma \ref{lem: covering number of V} below.

\begin{lemma}\label{lem: covering number of V}
    Assume $KH > 32$. For any $\epsilon<1$, the $\epsilon$-covering number of $\cV$ is upper bounded by 
    \begin{align*}
        \log \cN^{\cV}_{\epsilon} \leq 6 d^2 \log \left( 1 + \frac{dKH \min\{|\cI|, |\cJ|\} \cdot \beta_V }{\min\{\lambda, 1\} \cdot \epsilon  } \right) + 4 d^4 \log \left( 1 + \frac{d \min\{|\cI|, |\cJ| \} \cdot \beta_u  }{\min\{\lambda, 1\} \cdot \epsilon } \right).
    \end{align*} 
\end{lemma}

To prove Lemma \ref{lem: covering number of V}, we will first construct a function class $\cG$ that contains all $\hat{\PP}_h \bar{V}_{h+1}^k$, $\cR$ that contains all $\bar{r}_h^k$, and $\cQ$ that contains all $\bar{Q}_h^k$. The formal definition of these classes will be given in the following.

We also introduce the following technical lemma.

\begin{lemma}[Covering Number of $\ell_2$ Ball]
     For any $\epsilon >0$, the $\epsilon$-covering number of the $\ell_2$ ball in $\RR^d$ with radius $L$ is upper bound by $(1+2L/\epsilon)^d$.
\end{lemma} The proof of this classical result can be found in, for example, Chapter 5 in \citep{vershynin2010introduction}. Now we prove Lemma \ref{lem: covering number of V}.

\subsection{Proof of Lemma \ref{lem: covering number of V}}

\paragraph{Covering of $\hat{\PP}_h \bar{V}_{h+1}^k$.} The next lemma is helpful to bound the covering number of the function class containing the function $\hat{\PP}_h \bar{V}_{h+1}^k$. The proof is the same as that of Lemma D.6. in \citep{jin2020provably}. 

\begin{lemma}\label{lem: covering of regularized linear class}
    Let $\cG = \cG(L,B)$ denote the function class with functions mapping from $\cC \times \Upsilon$ to $\RR$ and of the following parametric form
    \begin{align*}
        g(\cdot,\cdot) = \bpsi(\cdot,\cdot)^\top \wb + \beta \cdot \left\| \bpsi(\cdot,\cdot) \right\|_{\bLambda^{-1}} \,, 
    \end{align*}where $\|\wb\|\leq L$, $\beta \in [0, B]$, and $\lambda_{\min}(\bLambda) \geq \lambda>0$. Assume $\|\bpsi(\cdot,\cdot)\|_2 \leq 1$. Let $\cN_\epsilon$ denote the $\epsilon$-covering number of $\cG$ with respect to the $\ell_\infty$ distance. Then we have
    \begin{align*}
        \log (\cN_\epsilon^{\cG}) \leq d \log(1 + 4L/\epsilon) + d^2 \log\left[ 1+ 8d^{1/2} B^2/(\lambda \epsilon^2) \right]. 
    \end{align*}
\end{lemma}

Suppose for now that there exists $L_{\wb}>0$ such that $\|\wb_h^k\|_2 \leq L_{\wb}$ for all $(k,h)$. The value of $L_{\wb}$ will be determined later.
By applying Lemma \ref{lem: covering of regularized linear class} with $L = L_{\wb}$ and $B = \beta_V$, we get the following upper bound on the $\epsilon$-covering number of the function class $\cG(l_{\wb}, \beta_V)$ which contains all $\hat{\PP}_h \bar{V}_{h+1}^k$:
\begin{align}\label{eq: covering number of P V}
    \log (\cN_\epsilon^{\cG}) \leq d \log(1 + 4L_{\wb}/\epsilon) + d^2 \log\left[ 1+ 8d^{1/2} \beta_V^2/(\lambda \epsilon^2) \right]. 
\end{align}

\paragraph{Covering of $\bar{r}_h^k$.} We now define a function class $\cR$ which provably contains all the pseudo-reward estimates $\bar{r}_h^k$. Formally speaking, by Algorithm \ref{alg: reward estimation}, the functions in $\cR$ are parametrized by the utility function estimates $u_h^k$ and $v_h^k$. Denote the functions class containing all these utility function estimates by $\cU$. Then according to Algorithm \ref{alg: utility estimation}, any function $u: \cC \times \Upsilon \times \cI \times \cJ \to \RR $ in $\cU$ can be written as
\begin{align}\label{eq: def of utility function class}
    u(C,e,i,j) = \left( \langle \bPhi(C,e,i,j) , \btheta \rangle + \beta_u \sqrt{\bPhi(C,e,i,j)^\top \bSigma^{-1} \bPhi (C,e,i,j)}\right)_{[-1,1]} .  
\end{align}where $\bPhi \in \RR^{d^2}$ satisfying $\| \bPhi\|_1 \leq 1$ by Assumption \ref{assump: model and feature}, $\|\btheta\| \leq L_u$ for some $L_u>0$ to be determined, and $\bSigma \in \RR^d\times\RR^d$ such that $\lambda_{\min}(\bSigma) \geq \lambda$. Since the truncation is a contraction mapping, by Lemma \ref{lem: covering of regularized linear class}, the $\epsilon$-covering number of $\cU$ is upper bounded by 
\begin{align}\label{eq: covering number of utility class original form}
    \log\left( \cN_\epsilon^u\right) & \leq d^2 \log\left( 1+4L_u/\epsilon \right) + d^4 \log \left[ 1 + 8d\beta_u^2 / (\lambda \epsilon^2) \right] .  
\end{align}
We now consider the function class $\cR$. We formally define $\cR$ to be the function class such that any function $r \in \cR$ can be represented by
\begin{align*}
    r(C,e,I,J) = \texttt{RE}(u(C,e,\cdot,\cdot), v(C,e,\cdot,\cdot), I,J)
\end{align*}for some $u,v \in \cU$. 
Let functions $r_1$, $r_2 \in \cR$ be parametrized by $u_1, v_1$ and $u_2, v_2$ respectively, such that  
\begin{align*}
    r_1(C,e,I,J) & = \texttt{RE}(u_1(C,e,\cdot,\cdot), v_1(C,e,\cdot,\cdot), I,J),
    \\ r_2(C,e,I,J) & = \texttt{RE}(u_2(C,e,\cdot,\cdot), v_2(C,e,\cdot,\cdot), I,J) ,
\end{align*} for all $C \in \cC$, $e\in \upsilon$, $I \subset \cI$ and $J \subset \cJ$. According to the linear program \ref{eq: linear program definition}, there exist some weights $w_1 = \{w_{1,i,j}\}_{(i,j)\in I\times J}$ and $w_2 =\{w_{2,i,j}\}_{(i,j)\in I\times J}$, such that
\begin{align*}
    r_1(C,e,I,J) & = \sum_{(i,j) \in I\times J} w_{1,i,j}\left[ u_1(C,e,i,j) + v_1(C,e,i,j)\right] , 
    \\ r_2(C,e,I,J) & = \sum_{(i,j) \in I\times J} w_{2,i,j}\left[ u_2(C,e,i,j) + v_2(C,e,i,j)\right] .
\end{align*} It follows that 
\begin{align*}
    (r_1-r_2)(C,e,I,J)  & = \sum_{(i,j)\in I\times J} w_{1,i,j} \left[ u_1(C,e,i,j) + v_1(C,e,i,j) \right] \notag
    \\ & \qquad \qquad - \sum_{(i,j)\in I\times J} w_{2,i,j} \left[ u_2(C,e,i,j) + v_2(C,e,i,j) \right] 
    \\ & \leq \sum_{(i,j)\in I\times J} w_{2,i,j} \left[ u_1(C,e,i,j) + v_1(C,e,i,j) \right] 
    \\ & \qquad \qquad - \sum_{(i,j)\in I\times J} w_{2,i,j} \left[ u_2(C,e,i,j) + v_2(C,e,i,j) \right]
    \\ & \leq \sum_{(i,j)\in I\times J} w_{2,i,j} \left[ \left( u_1(C,e,i,j)-u_2(C,e,i,j) \right) + \left( v_1(C,e,i,j)-v_2(C,e,i,j)  \right)\right] 
    \\ & \leq \min\{|I|, |J|\} \cdot \left( \|u_1 - u_2 \|_\infty + \| v_1 - v_2\|_\infty \right),
\end{align*}where the second step holds because $w_1$ is the optimal weight given $u_1$ and $v_1$ and $w_2$ satisfies the constraint of the linear program with $u_1$ and $v_1$. The same upper bound holds for the difference $(r_2- r_1)$. Therefore, for any $I,J$, we have 
\begin{align*}
    \left\| r_1(\cdot, \cdot, I,J) - r_2(\cdot,\cdot, I,J)\right\|_\infty \leq \min\{|I|, \ |J|\} \cdot \left( \|u_1-u_2\|_\infty + \|v_1 - v_2\|_\infty \right). 
\end{align*} Since $I \subset \cI$ and $J \subset \cJ$, in order for $\|r_1 - r_2\|_{\infty} \leq \epsilon$ to hold, it suffices to have $\|u_1 - u_2\|_\infty \leq \epsilon'$ and $\|v_1 - v_2\|_\infty \leq \epsilon'$ where $\epsilon' = \epsilon/(2 \min\{|\cI|, |\cJ|\})$.
Therefore, by \eqref{eq: covering number of utility class original form}, the $\epsilon$-covering number of $\cR$ satisfies
\begin{align}\label{eq: covering number of R}
    \log \cN_{\epsilon}^{\cR} \leq 2 \log \cN_{\epsilon'}^u & \leq 2 d^2 \log\left( 1+4L_u/\epsilon' \right) + 2 d^4 \log \left[ 1 + 8d\beta_u^2 / (\lambda {\epsilon'}^2) \right] \notag 
    \\ & \leq 2 d^2 \log\left( 1 + \frac{8L_u \min\{|\cI|, |\cJ|\} }{\epsilon} \right) + 2d^4 \log \left[ 1 + \frac{32 d \beta_u^2 \left(\min\{|\cI|, |\cJ|\}\right)^2 }{\lambda \epsilon^2} \right] .
\end{align}

In the above we have shown that the function class $\cR$ contains all $\bar{r}_h^k$ and $\cG$ contains all $\hat{\PP}_h \bar{V}_{h+1}^k$. 
We now define the function class $\cQ \coloneqq \cR + \cG$ as
\begin{align*}
    \cQ \coloneqq \left\{(r+g)_{[0, \sum_{l=h}^H W_l]} \middle| \  r \in \cR, \ g \in \cG\right\}. 
\end{align*}
Then it immediately follows from the algorithm that $\cQ$ contains all $\bar{Q}_h^k$ functions. 
By \eqref{eq: covering number of P V} and \eqref{eq: covering number of R}, the $\epsilon$-covering number of the function class $\cQ$ can be upper bounded by 
\begin{align}\label{eq: covering Q}
    \log \cN_{\epsilon}^{\cQ} & \leq d \log(1 + 4L_{\wb}/\epsilon) + d^2 \log\left[ 1+ 8d^{1/2} \beta_V^2/(\lambda \epsilon^2) \right] \notag
    \\ 
    & \qquad + 2 d^2 \log\left( 1 + \frac{8L_u \min\{|\cI|, |\cJ|\} }{\epsilon} \right) \notag
    \\ 
    & \qquad + 2d^4 \log \left[ 1 + \frac{32 d \beta_u^2 \left(\min\{|\cI|, |\cJ|\}\right)^2 }{\lambda \epsilon^2} \right] ,
\end{align}
Since by construction, $\bar{V}_h^k(C,I,J) = \max_{e} \bar{Q}_h^k(C,e,I,J)$ and taking the maximum is a contraction mapping, the upper bound in \eqref{eq: covering Q} also holds for $\log \cN_{\epsilon}^{\cV}$. 

By Lemma \ref{lem: bound of regression estimator of PV}, we can pick 
\begin{align*}
    L_{\wb} = \left(\sum_{h=1}^{H} W_h\right) \cdot \sqrt{dK/\lambda}\qquad \text{and} \qquad L_u = \sqrt{\frac{d^2 K \cdot \min\{|\cI|, |\cJ| \}}{\lambda}}.
\end{align*}

From the above analysis, we can simplify the R.H.S. of \eqref{eq: covering Q} and get the desired bound for the covering number of $\cV$. This finishes the proof of Lemma \ref{lem: covering number of V}.

\section{Auxiliary Lemmas}

\begin{lemma}[Lemma D.1 in \citealt{jin2020provably}]\label{lem: matrix inequality gram sum}
    For arbitrary $d$, let $\bLambda_k = \lambda \bI_d + \sum_{t=1}^{k-1} \xb_t \xb_t^\top $ where $\xb_t \in \RR^d$ and $\lambda>0$. Then 
    \begin{align*}
        \sum_{t=1}^{k-1} \xb_t^\top \left( \bLambda_k\right)^{-1} \xb_t \leq d.
    \end{align*}
\end{lemma}

\begin{proof}[Proof of Lemma \ref{lem: matrix inequality gram sum}]
    We can write
    \begin{align*}
        \sum_{t=1}^{k-1} \xb_t^\top \left( \bLambda_k\right)^{-1} \xb_t & = \sum_{t=1}^{k-1} \textnormal{tr}\left(\xb_t^\top \left( \bLambda_k\right)^{-1} \xb_t\right) = \textnormal{tr}\Big( \left( \bLambda_k\right)^{-1} \sum_{t=1}^{k-1} \xb_t^\top \xb_t\Big).
    \end{align*}Denote the eigenvalue of $\sum_{t = 1}^{k-1} \xb_t^{\top} \xb_t$ as $\{\lambda_1,\cdots, \lambda_d\}$, and decompose $\sum_{t=1}^{k-1} \xb_t^\top \xb = \bU \textnormal{diag}(\lambda_1,\cdots, \lambda_d) \bU^\top$. Then we have $\bLambda_k = \bU \textnormal{diag}(\lambda_1+\lambda,\cdots, \lambda_d+\lambda) \bU^\top$. It follows that $\textnormal{tr}( ( \bLambda_k)^{-1} \sum_{t=1}^{k-1} \xb_t^\top \xb_t) = \sum_{j=1}^d \lambda_j/(\lambda_j + \lambda) \leq d$. 
\end{proof}

The next is the well-known Elliptical Potential Lemma \citep{cesa2006prediction, abbasi2011improved, lattimore2020bandit}.

\begin{lemma}[Elliptical Potential Lemma]\label{lem: elliptical potential}
    For arbitrary $d$, let $\bLambda_k = \lambda \bI_d + \sum_{t=1}^{k-1} \xb_t \xb_t^\top $ where $\xb_t \in \RR^d$ and $\lambda>0$. Then 
    \begin{align*}
        \sum_{t=1}^{k} \left\|\xb_t \right\|_{(\bLambda_{t+1})^{-1}} \leq \sqrt{ kd \log\left( \frac{k + d\lambda}{d\lambda}\right) }.
    \end{align*}
\end{lemma}

\begin{lemma}[Azuma-Hoeffding inequality \citep{azuma1967weighted}]\label{lem: azuma hoeffding vanilla} 
Let $\{X_t\}_{t=0}^\infty$ be a real-valued martingale such that for every $t \geq 1$, it holds that $|X_t - X_{t-1}| \leq B_t$ for some $B_t \geq 0$. Then 
\begin{align*}
    \PP\left( \left| X_t - X_0\right| \geq \epsilon \right) \leq 2 \exp\left( \frac{- \epsilon^2}{2 \sum_{\tau = 1}^t B_{\tau}^2} \right) .
\end{align*}
\end{lemma}

\subsection{Concentration Inequalities for Self-normalized Martingales}
\begin{theorem}[Hoeffding inequality for Self-normalized martingales \citep{abbasi2011improved}]\label{thm: hoeffding self normalized}
    Let $\{\eta_t\}_{t=1}^\infty$ be a real-valued stochastic process. Let $\{\cF_t \}_{t=0}^\infty$ be a filtration, such that $\eta_t$ is $\cF_t$-measurable. Assume $\eta_t \mid \cF_{t-1}$ is zero-mean and $R$-subgaussian for some $R > 0$, i.e.,
\begin{align*}
    \forall \lambda \in \RR, \quad \EE\left[ e^{\lambda \eta_t \mid \cF_{t-1}} \right] \leq e^{\lambda^2 R^2 /2}.
\end{align*}
Let $\{\xb_t\}_{t=1}^\infty$ be an $\RR^d$-valued stochastic process where $\xb_t$ is $\cF_{t-1}$-measurable. 
Assume $\bLambda_0$ is a $d\times d$ positive definite matrix, and let $\bLambda_t = \bLambda_0 + \sum_{s=1}^t \xb_s \xb_s^\top$. 
Then, for any $\delta >0$, with probability at least $1-\delta$, for all $t > 0$, 
\begin{align*}
     \left\|\sum_{s=1}^t \xb_s \eta_s \right\|^2_{\bLambda_t^{-1}} \leq 2 R^2 \log \left( \frac{ \det(\bLambda_t)^{1/2} \det(\bLambda_0)^{-1/2} }{\delta}\right). 
\end{align*}
\end{theorem}

\begin{lemma}[Lemma D.4 in \citealt{jin2020provably}]\label{lem: self normalized for V function class raw}
    Let $\cV$ be a function class such that any $V\in \cV$ maps from $\cS \to \RR$ and $\|V\|_{\infty} \leq R$.
    Let $\{\cF_t\}_{t=0}^\infty$ be a filtration. Let $\{s_t\}_{t=1}^\infty$ be a stochastic process in the space $\cS$ such that $s_t$ is $\cF_{t}$-measurable. Let $\{\xb\}_{t=0}^\infty$ be an $\RR^d$-valued stochastic process such that $\xb_t$ is $\cF_{t-1}$-measurable and $\|\xb\|_2 \leq 1$. Let $\bLambda_k = \lambda \bI + \sum_{t=1}^{k-1} \xb_t \xb_t^\top$. Then for any $\delta>0$, with probability at least $1-\delta$, for any $k$, and any $V \in \cV$, we have 
    \begin{align*}
        \left\| \sum_{t=1}^{k-1} \xb_t \left[ V(s_t) - \EE\left[V(s_t) \mid \cF_{t-1} \right]\right] \right\|^2_{(\bLambda_k)^{-1}} & \leq 4 R^2 \left[ \frac{d}{2}\log\left(\frac{k+\lambda}{\lambda} \right) + \log \frac{\cN^{\cV}_\epsilon}{\delta} \right] + \frac{8 k^2 \epsilon^2}{\lambda},
    \end{align*}where $\cN^{\cV}_\epsilon$ is the $\epsilon$-covering number of $\cV$ with respect to the $\ell_\infty$ distance.
\end{lemma}

\begin{proof}[Proof of Lemma \ref{lem: self normalized for V function class raw}]
    For any $V \in \cV$, there exists $V'$ in the $\epsilon$-covering such that 
    \begin{align*}
        V = V' + \Delta_V \quad \textnormal{and} \quad \left\| \Delta_V \right\|_{\infty} \leq \epsilon.
    \end{align*}
    Then we have 
    \begin{align*}
        \left\| \sum_{t=1}^{k-1} \xb_t \left[ V(s_t) - \EE\left[V(s_t) \mid \cF_{t-1} \right]\right] \right\|^2_{(\bLambda_k)^{-1}} & \leq 2 \left\| \sum_{t=1}^{k-1} \xb_t \left[ V'(s_t) - \EE\left[V'(s_t) \mid \cF_{t-1} \right]\right] \right\|^2_{(\bLambda_k)^{-1}} 
        \\ 
        & \qquad + 2 \left\| \sum_{t=1}^{k-1} \xb_t \left[ \Delta_V(s_t) - \EE\left[\Delta_V(s_t) \mid \cF_{t-1} \right]\right] \right\|^2_{(\bLambda_k)^{-1}} . 
    \end{align*} For the first term on the R.H.S., we apply Theorem \ref{thm: hoeffding self normalized} and a union bound to the $\epsilon$-covering. The second term can be bound by $8k^2 \epsilon^2/\lambda$ by using $\|\xb_t\|_2\leq 1$, $\lambda_{\min}(\bLambda_k)\geq \lambda$ and $\|\Delta_V\|_\infty \leq \epsilon$. 
    
\end{proof}

\end{document}